\documentclass{article}

\usepackage{microtype}
\usepackage{graphicx, caption, subcaption}
\usepackage{booktabs} 
\usepackage{ragged2e}

\usepackage[colorlinks,linkcolor=black,citecolor=blue, pagebackref=true]{hyperref}
\renewcommand*{\backrefalt}[4]{%
    \ifcase #1 \footnotesize{(Not cited.)}%
    \or        \footnotesize{(Cited on page~#2.)}%
    \else      \footnotesize{(Cited on pages~#2.)}%
    \fi}
\usepackage{bookmark}

\usepackage[accepted]{icml2024}

\usepackage{amsmath}
\usepackage{amssymb}
\usepackage{mathtools}
\usepackage{amsthm}

\usepackage[capitalize,noabbrev]{cleveref}

\theoremstyle{plain}
\newtheorem{theorem}{Theorem}[section]
\newtheorem{proposition}[theorem]{Proposition}
\newtheorem{lemma}[theorem]{Lemma}

\theoremstyle{definition}
\newtheorem{definition}[theorem]{Definition}

\theoremstyle{remark}

\usepackage[textsize=tiny]{todonotes}

\usepackage{eqnarray,amsmath}
\usepackage[utf8]{inputenc} 
\usepackage[T1]{fontenc}    
\usepackage{booktabs}       
\usepackage{amsfonts}       
\usepackage{nicefrac}       
\usepackage{microtype}      

\usepackage{caption}
\usepackage{subcaption}

\usepackage{epsf}
\usepackage{epsfig}
\usepackage{fancyhdr}
\usepackage{graphics}
\usepackage{graphicx}
\usepackage{psfrag}

\usepackage{url}

\usepackage{color}

\usepackage{amsthm}
\usepackage{amsmath}
\usepackage{amssymb,bbm}
\usepackage[justification=centering]{caption}
\usepackage{algorithmic}
\usepackage{algorithm}
\usepackage{textcomp}
\usepackage{siunitx}
\usepackage{wrapfig}
\usepackage{algorithmic}
\usepackage{algorithm}
\usepackage{multirow}
\usepackage{multicol}

\newcommand{\bbP}{\mathbb{P}}
\newcommand{\bbE}{\mathbb{E}}

\DeclareMathOperator*{\argmax}{arg\,max}
\DeclareMathOperator*{\argmin}{arg\,min}

\def\wrt{w.r.t.~}

\newcommand{\dboijn}{\Delta\beta_{1ij}^{n}}
\newcommand{\daijnl}{\Delta a_{ij\ell}^{n}}
\newcommand{\dbijnl}{\Delta b_{ij\ell}^{n}}

\newcommand{\bzin}{\beta^{n}_{0i}}
\newcommand{\boin}{\beta^{n}_{1i}}
\newcommand{\ain}{a_{i}^n}
\newcommand{\bin}{b_{i}^n}

\newcommand{\bzj}{\beta_{0j}^{*}}
\newcommand{\boj}{\beta_{1j}^{*}}
\newcommand{\aj}{a_{j}^{*}}
\newcommand{\bj}{b_{j}^{*}}

\newcommand{\zerod}{\mathbf{0}_d}

\newcommand{\dint}{\mathrm{d}}
\newcommand{\gmod}{\widetilde{g}}
\newcommand{\tmod}{\widetilde{T}}
\newcommand{\umod}{\widetilde{u}}
\newcommand{\vmod}{\widetilde{v}}
\newcommand{\rmod}{\widetilde{R}}

\newcommand{\softmax}{\mathrm{Softmax}}

\icmltitlerunning{A General Theory for Softmax Gating Multinomial Logistic Mixture of Experts}

\begin{document}

\twocolumn[
\icmltitle{A General Theory for Softmax Gating Multinomial Logistic Mixture of Experts}



\icmlsetsymbol{equal}{*}

\begin{icmlauthorlist}
\icmlauthor{Huy Nguyen}{sds}
\icmlauthor{Pedram Akbarian}{ece}
\icmlauthor{TrungTin Nguyen}{uq,uga}
\icmlauthor{Nhat Ho}{sds}
\end{icmlauthorlist}

\icmlaffiliation{sds}{Department of Statistics and Data Sciences}
\icmlaffiliation{ece}{Department of Electrical and Computer Engineering, The University of Texas at Austin}
\icmlaffiliation{uga}{Univ. Grenoble Alpes, Inria, CNRS, Grenoble INP, LJK}
\icmlaffiliation{uq}{School of Mathematics and Physics, The University of Queensland}

\icmlcorrespondingauthor{Huy Nguyen}{huynm@utexas.edu}

\icmlkeywords{Machine Learning, ICML}

\vskip 0.3in
]



\printAffiliationsAndNotice{}  

\begin{abstract}
    Mixture-of-experts (MoE) model incorporates the power of multiple submodels via gating functions to achieve greater performance in numerous regression and classification applications. From a theoretical perspective, while there have been previous attempts to comprehend the behavior of that model under the regression settings through the convergence analysis of maximum likelihood estimation in the Gaussian MoE model, such analysis under the setting of a classification problem has remained missing in the literature. We close this gap by establishing the convergence rates of density estimation and parameter estimation in the softmax gating multinomial logistic MoE model. Notably, when part of the expert parameters vanish, these rates are shown to be slower than polynomial rates owing to an inherent interaction between the softmax gating and expert functions via partial differential equations. To address this issue, we propose using a novel class of modified softmax gating functions which transform the input before delivering them to the gating functions. As a result, the previous interaction disappears and the parameter estimation rates are significantly improved.
\end{abstract}
\section{Introduction}
\label{sec:introduction}
Mixture of experts (MoE) \cite{Jacob_Jordan-1991,Jordan-1994} is a statistical machine learning model which aggregates several submodels represented by expert networks associated with standard softmax gating functions. 
Due to its modular and flexible structure, there has been a surge of interests in leveraging the MoE model in various fields, including large language models \cite{shazeer2017sparse,lepikhin_gshard_2021,zhou2022expertchoice,du2022glam,fedus2022swtich,zhou2023brainformers,do_hyperrouter_2023}, computer vision \cite{deleforge_high_dimensional_2015,Riquelme2021scalingvision,dosovitskiy_image_2021,bao_vlmo_2022,liang_m3vit_2022}, speech recognition \cite{peng_bayesian_1996,gulati_conformer_2020,You_Speech_MoE_2}, reinforcement learning \cite{ren2021probabilistic,chow_mixture_expert_2023} and multi-task learning \cite{hazimeh2021multitask,gupta2022sparsely,chen2023modsquad}. Despite its recent progress in the aforementioned applications, the theoretical comprehension of the MoE model has been found relatively restricted in the literature. 

In general, that model can be formulated as either a regression problem, namely when the distribution of MoE outputs is continuous, or a classification problem, i.e., when the MoE outputs follow a discrete distribution. Under the setting of a regression problem, there are some previous works attempting to theoretically understand the MoE model through the convergence analysis of maximum likelihood estimation (MLE) by focusing on Gaussian MoE equipped with different types of gating functions. First, \cite{ho2022gaussian} established the convergence rates of parameter estimation in the covariate-free gating Gaussian MoE model using the generalized Wasserstein loss function \cite{Villani-03,Villani-09}. In that paper, they discovered that those parameter estimation rates were significantly slow owing to an interaction between expert parameters via some partial differential equations (PDEs). Next, \cite{nguyen2024gaussian} considered the Gaussian density gating Gaussian MoE model. Since that gating function depended on the covariates, there was another interaction between the parameters of the gating function and the expert function. To this end, the authors proposed Voronoi loss functions to capture those interactions and demonstrate that the convergence rates of parameter estimation were determined by the solvability of a system of polynomial equations. Subsequently, the Gaussian MoE models with softmax gating function and top-K sparse softmax gating function were investigated in \cite{nguyen2023demystifying} and \cite{nguyen2024statistical}, respectively. Due to the non-linearity and sophisticated structures of those gating functions, the parameter estimation rates were shown to vary with the solvability of a more complex system of polynomial equations than that in \cite{nguyen2024gaussian}, and became substantially slow when the number of fitted experts increased. 

On the other hand, such theoretical analysis of the MoE model under the setting of a classification problem has remained poorly understood, to the best of our knowledge. Therefore, the main objective of our paper is to provide new insights on the convergence behavior of MLE in the softmax gating multinomial logistic MoE model \cite{chen1999multiclass,yuksel2010vmoe,huynh_estimation_2019,pham_functional_2022}. Before going into further details, it is necessary to introduce the formulation of that model and associated assumptions.

\textbf{Problem Setup.} In this paper, we assume that the output $Y\in\{1,2,\ldots,K\}$ is a discrete response variable, where $K\in\mathbb{N}$, while $X\in\mathcal{X}$ is an covariate vector having an effect on $Y$, in which $\mathcal{X}$ is a compact subset of $\mathbb{R}^d$. Next, the data points $(X_1,Y_1),(X_2,Y_2),\ldots,(X_n,Y_n)$ are independently drawn from the standard softmax gating multinomial logistic mixture of experts of order $k_*$, which admits the conditional probability function $g_{G_{*}}(Y=s|X)$ defined for any $s\in\{1,2,\ldots,K\}$ as follows:
\begin{align}
    &\sum_{i=1}^{k_*}~\softmax((\beta^*_{1i})^{\top}X+\beta^*_{0i})\times f(Y=s|X;a^*_{i},b^*_{i})
    \nonumber\\
    &:=\sum_{i=1}^{k_*}~\frac{\exp((\beta^*_{1i})^{\top}X+\beta^*_{0i})}{\sum_{j=1}^{k_*}\exp((\boj)^{\top}X+\bzj)}\nonumber\\
    \label{eq:conditional_density}
    &\hspace{2.2cm}\times\frac{\exp(a^*_{is}+(b^*_{is})^{\top}X)}{\sum_{\ell=1}^{K}\exp(a^*_{i\ell}+(b^*_{i\ell})^{\top}X)}.
\end{align}
Here, each expert $f(\cdot|X;a^*_i,b^*_i)$ is a multinomial logistic regression with parameters $a^*_i:=(a^*_{i1},\ldots,a^*_{iK})\in\mathbb{R}^K$ and $b^*_i:=(b^*_{i1},\ldots,b^*_{iK})\in\mathbb{R}^{d\times K}$. Meanwhile, $\beta^*_{1i}\in\mathbb{R}^d$ and $\beta^*_{0i}\in\mathbb{R}$ are referred to as gating parameters. Additionally,  $G_*:=\sum_{i=1}^{k_*}\exp(\beta^*_{0i})\delta_{(\beta^*_{1i},a^*_{i},b^*_{i})}$ denotes a true yet unknown \emph{mixing measure}, that is, a combination of Dirac measures $\delta$ associated with true parameters $(\beta^*_{0i},\beta^*_{1i},a^*_{i},b^*_{i})\in\Theta$, where $\Theta$ is a bounded subset of $\mathbb{R}\times\mathbb{R}^d\times\mathbb{R}^{ K}\times\mathbb{R}^{d\times K}$. Lastly, we define for any vector $v=(v_1,\ldots,v_{k_*})\in\mathbb{R}^{k_*}$ that $\softmax(v_i):={\exp(v_i)}/{\sum_{j=1}^{k_*}\exp(v_j)}$.

It is worth noting that if we translate $\beta^*_{1i}$ to $\beta^*_{1i}+t_1$ and $\beta^*_{0i}$ to $\beta^*_{0i}+t_0$, then the values of the standard softmax gating function do not change. This implies that gating parameters $\beta^*_{1i},\beta^*_{0i}$ are only identifiable up to some translation. To alleviate this problem, we assume without loss of generality (WLOG) that $\beta^*_{1k_*}=\zerod$ and $\beta^*_{0k_*}=0$. Similarly, we also assume that $a^*_{iK}=0$ and $b^*_{iK}=\zerod$ for any $i\in\{1,2,\ldots,k_*\}$. Furthermore, at least one among $\beta^*_{11},\beta^*_{12},\ldots,\beta^*_{1k_*}$ must be non-zero so that the softmax gating function depends on the covariate $X$. Lastly, we let $(a^*_{i},b^*_{i})$, for $i\in\{1,2,\ldots,k_*\}$, be pairwise distinct to ensure that the epxert functions are different from each other. 

\textbf{Maximum Likelihood Estimation.} Regarding the parameter estimation problem in the standard softmax gating multinomial logistic MoE, we propose using the maximum likelihood method in this work. However, as the true number of experts $k_*$ is generally unknown in practice, it is necessary to consider an \emph{over-specified} setting where we fit the true model with a mixture of $k$ multinomial logistic experts, where $k>k_*$. In particular, the maximum likelihood estimation (MLE) of the true mixing measure $G_*$ is given by:
\begin{align}
    \label{eq:MLE}
    \widehat{G}_n\in\argmax_{G\in\mathcal{O}_k(\Theta)}\sum_{i=1}^{n}\log(g_{G}(Y_i|X_i)),
\end{align}
where $\mathcal{O}_k(\Theta)$ denotes the set of all mixing measures with at most $k$ components of the form $G=\sum_{i=1}^{k'}\exp(\beta_{0i})\delta_{(\beta_{1i},a_{i},b_{i})}$, in which $1\leq k'\leq k$ and $(\beta_{0i},\beta_{1i},a_{i},b_{i})\in\Theta$.

\textbf{Main Challenges.} To characterize the parameter estimation rates, we need to decompose the difference  $g_{\widehat{G}_n}(Y|X)-g_{G_*}(Y|X)$ into a linear combination of linearly independent terms using Taylor expansions. Then, when the density estimation $g_{\widehat{G}_n}(Y|X)$ converges to the true density $g_{G_*}(Y|X)$, all the associated coefficients, which involve the discrepancies between true parameters and their estimations, also tend to zero. Consequently, we achieve our desired parameter estimation rates based on the density estimation rate. However, when part of the expert parameters vanishes, there is an interaction between the numerator of the standard softmax gating function and the expert function, which induces several challenges in the density decomposition. In particular, let us denote $u(Y=s|X;\beta_{1i},a_{i},b_{i}):=\exp((\beta_{1i})^{\top}X)f(Y=s|X;a_i,b_i)$. If there exists an index $i\in[k_*]$ such that $b^*_{i\ell}=\zerod$ for any $\ell\in[K]$, then the aforementioned interaction is expressed via the following PDE:
\begin{align}
    &\frac{\partial u}{\partial\beta_{1i}}(Y=s|X;\beta^*_{1i},a^*_{i},b^*_{i})\nonumber\\
    \label{eq:PDE}
    &\hspace{1.5cm}=C_{a^*_i}\cdot\frac{\partial u}{\partial b_{is}}(Y=s|X;\beta^*_{1i},a^*_{i},b^*_{i}),
\end{align}
where $C_{a^*_i}>0$ is a constant depending only on $a^*_i$. The above PDE shows that there are a number of linearly dependent derivative terms in the Taylor expansion. Then, we have to incorporate these terms by taking the summation of their associated coefficients in order to form a linear combination of linearly independent terms. Therefore, the structure of resulting coefficients becomes complex, which makes the parameter estimation rates slower than polynomial rates. This finding indicates that the standard softmax gating function might hurt the performance of the model in equation~$\eqref{eq:conditional_density}$ despite its widespread use in the literature \cite{Jacob_Jordan-1991,Jordan-1994,nguyen2023demystifying}. 

\begin{table*}[!ht]
\centering
\caption{Summary of density estimation rates and parameter estimation rates in the multinomial logistic MoE model with standard and modified softmax gating functions. Here, we refer to exact-specified parameters $\beta^*_{1j},a^*_j,b^*_j$ as those fitted by exactly one component (i.e. $|\mathcal{C}_j|=1$), while over-specified parameters are approximated by more than one component (i.e. $|\mathcal{C}_j|>1$), where $\mathcal{C}_j$ is a Voronoi cell defined in Section~\ref{sec:preliminaries}.}
\vspace{0.5em}
\begin{tabular}{ | m{4em} | m{5em}| c | c |c|} 
\hline
\textbf{Gating} & \textbf{Parameter Setting} & \textbf{Density} &  \textbf{Exact-specified Parameters} & \textbf{Over-specified Parameters} \\
\hline 
\multirow{2}{4em}{Standard Softmax} & Regime 1 & $\widetilde{\mathcal{O}}(n^{-1/2})$ & $\widetilde{\mathcal{O}}(n^{-1/2})$ & $\widetilde{\mathcal{O}}(n^{-1/4})$ \\  \cline{2-5}
 & Regime 2 & $\widetilde{\mathcal{O}}(n^{-1/2})$ & \multicolumn{2}{c|}{slower than $\widetilde{\mathcal{O}}(n^{-1/2r})$ for any $ r\geq 1$}  \\ 
\hline
Modified Softmax & Any regime & $\widetilde{\mathcal{O}}(n^{-1/2})$ & $\widetilde{\mathcal{O}}(n^{-1/2})$ & $\widetilde{\mathcal{O}}(n^{-1/4})$ \\  
\hline
\end{tabular}
\label{table:parameter_rates}
\end{table*}

\textbf{Contribution.} Following the above challenge discussion, we construct a generic Voronoi-based metric $\mathcal{D}_r$ among parameters in equation~\eqref{eq:Voronoi_loss_form} to capture the convergence behavior of the MLE $\widehat{G}_n$ to the true mixing measure $G_*$ in the standard softmax gating multinomial logistic MoE model. Moreover, we also design a novel class of modified softmax gating functions which improves the convergence rate of the MLE. Our contributions are two-fold and can be summarized as follows (see also Table~\ref{table:parameter_rates}):

\textbf{1. Standard Softmax Gating Function.} With the standard softmax gating function, we demonstrate that the density estimation $g_{\widehat{G}_n}$ converges to the true density $g_{G_*}$ under the Hellinger distance $h$ at the parametric rate of order $\widetilde{\mathcal{O}}(n^{-1/2})$. Next, we consider the parameter estimation problem under two following complement regimes of the expert parameters $b^*_i$:

\textbf{(i) Regime 1:} For any $i\in\{1,2,\ldots,k_*\}$, we can find an index $\ell\in[K-1]$ that satisfies $b^*_{i\ell}\neq\zerod$. Under this regime, the PDE in equation~\eqref{eq:PDE} does not hold and there is no interaction between the standard softmax gating and the expert functions. By deriving the Hellinger lower bound $\bbE_X[h(g_{G}(\cdot|X),g_{G_*}(\cdot|X))]\gtrsim\mathcal{D}_2(G,G_*)$ for any mixing measure $G\in\mathcal{O}_k(\Theta)$ in Theorem~\ref{theorem:standard_one}, we obtain that the rates for estimating over-specified parameters $\beta^*_{1i}$, $a^*_{i}$ and $b^*_{i}$, i.e. those whose Voronoi cells have more than one element, are identical of order $\widetilde{\mathcal{O}}(n^{-1/4})$. At the same time, the estimation rates for their exact-specified parameters, namely those whose Voronoi cells have exactly one element, are substantially faster of order $\widetilde{\mathcal{O}}(n^{-1/2})$. 

\textbf{(ii) Regime 2:} There exists an index $i\in\{1,2,\ldots,k_*\}$ such that $b^*_{i\ell}=\zerod$ for any $\ell\in[K-1]$. Since the PDE~\eqref{eq:PDE} holds true under this regime, an interaction between the standard softmax gating  and the expert functions occurs. Then, we demonstrate in Theorem~\ref{theorem:standard_two} that the minimax lower bound $\inf_{\overline{G}_n\in\mathcal{O}_{k}(\Theta)}\sup_{G\in\mathcal{O}_{k}(\Theta)\setminus\mathcal{O}_{k_*-1}(\Theta)}\bbE_{p_{G}}[\mathcal{D}_r(\overline{G}_n,G)]\gtrsim n^{-1/2}$ holds for any $r\geq 1$. This bound together with the formulation of $\mathcal{D}_r$ indicate that true parameters $\beta^*_{1i}$, $a^*_i$ and $b^*_i$ enjoy much worse rates than those in Regime 1. Remarkably, those for over-specified parameters are even slower than polynomial rates.  

\textbf{2. Modified Softmax Gating Function.} To resolve the aforementioned downsides of the standard softmax gating function towards the convergence rates of parameter estimation, we propose using the following modified softmax gating functions: 
\begin{align*}
    \softmax(&(\beta^*_{1i})^{\top}M(X)+\beta^*_{0i})\\
    &:=\frac{\exp((\beta^*_{1i})^{\top}M(X)+\beta^*_{0i})}{\sum_{j=1}^{k_*}\exp((\boj)^{\top}M(X)+\bzj)},
\end{align*}
for any $i\in\{1,2,\ldots,k_*\}$ where $M:\mathbb{R}^d\to\mathbb{R}^d$ is a bounded function such that the set $\Big\{X^p[M(X)]^{q}: p,q\in\mathbb{N}^d, \ 0\leq|p|+|q|\leq 2\Big\}$ is linearly independent for almost surely $X$. These assumptions guarantee that the PDE in equation~\eqref{eq:PDE} does not hold for any values of expert parameters $b^*_i$, and therefore, the previous interaction between the gating and expert functions disappears. As a consequence, we show in Theorem~\ref{theorem:modified_softmax} that the parameter estimations under the modified softmax gating multinomial logistic MoE model share the same convergence behavior as those in Regime 1 regardless of the values of expert parameters $b^*_i$. 

\textbf{Practical Implications.} Below are three practical implications from our convergence analysis of the softmax gating multinomial logistic MoE:

\textbf{1. Expert parameter collapse:} It is worth noting that the parameter estimation rates when the expert parameters collapse are significantly slow, and could be of order $O(1/\log(n))$ (see Table 1). Therefore, during the training process, if ones observe that the model convergence becomes abnormally slow, or the updated loss values almost remain unchanged, then it is highly likely that the expert parameter collapse occurs.

\textbf{2. Model design:} Despite the widespread use of the standard softmax gate in the practical applications of MoE models, the insights from our theory indicates that this gate is not always beneficial for the model performance, particularly when the expert parameter collapse happens. Therefore, our analysis suggests using a novel modified softmax gate which helps stabilize the training process regardless of the parameter collapse. 

\textbf{3. Expert selection:} The practical problem of selecting important experts can benefit from the modified softmax gate. In particular, in many applications of MoE models, there are some experts which do not play an essential role in learning, and even become redundant. In such scenario, we would like the gate to put more weights on the important experts. However, if the input magnitude is huge, then the weight distribution will become uniform, which is undesirable. To address this issue, practitioners can use an input normalization function, i.e. $M(X)=\frac{X}{||X||}$, in the softmax weights so that the input magnitude remains unchanged of one. Another possible option is $M(X)=sigmoid(X)$, which allows the magnitude of the input to vary between 0 and 1.

\textbf{Outline.} The remainder of the paper is organized in the following way. In Section~\ref{sec:preliminaries}, we verify the identiability of the standard softmax gating multinomial logistic MoE model and establish the parametric density estimation rate under that model prior to introducing a novel class of Voronoi loss functions used for the parameter estimation problem. Subsequently, we characterize the parameter estimation rates under the multinomial logistic MoE model equipped with the standard softmax gating function and the modified softmax gating function in Section~\ref{sec:standard} and Section~\ref{sec:modified}, respectively. Finally, we conclude the paper in Section~\ref{sec:conclusion}. Meanwhile, full proofs, additional results and a simulation study are relegated to the Appendices. 

\textbf{Notations.} First, we let $[n]$ stand for the set $\{1,2,\ldots,n\}$ for any number $n\in\mathbb{N}$. Next, for any vector $u,v\in\mathbb{N}^d$, we denote $|v|:=v_1+v_2+\ldots+v_d$, $v^u:=v_1^{u_1}v_2^{u_2}\ldots v_d^{u_d}$ and $v!:=v_1!v_2!\ldots v_d!$, whereas $\|v\|$ represents for its $2$-norm value. Meanwhile, for any set $S$, we denote $|S|$ as its cardinality. Additionally, for any two probability density functions $p,q$ dominated by the Lebesgue measure $\mu$, we define $V(p,q):=\frac{1}{2}\int|p-q|\dint \mu$ as the Total Variation distance between them, and $h(p,q):=\Big(\frac{1}{2}\int|\sqrt{p}-\sqrt{q}|^2\dint \mu\Big)^{1/2}$ as their Hellinger distance. Lastly, for any two sequences $(s_n)$ and $(t_n)$ in $\mathbb{R}_+$, the notations $s_n=\mathcal{O}(t_n)$ and $s_n\lesssim t_n$ means there exists some constant $C>0$ independent of $n$ such that $s_n\leq Ct_n$ for sufficiently large $n\in\mathbb{N}$, while the notation $s_n=\widetilde{\mathcal{O}}(t_n)$ suggests that the aforementioned inequality might hold up to some logarithmic function of $n$.
\section{Preliminaries}
\label{sec:preliminaries}
In this section, we study the identifiability of the standard softmax gating multinomial logistic MoE model and the convergence behavior of density estimation under that model. Then, we further highlight the necessity for Voronoi-based loss functions to determine parameter estimation rates accurately.

Firstly, we show that the standard softmax gating multinomial logistic MoE model is identifiable.
\begin{proposition}[Identifiability]
    \label{prop:identifiability}
    Given two mixing measures $G$ and $G'$ in $\mathcal{O}_k(\Theta)$, if $g_{G}(Y|X)=g_{G'}(Y|X)$ holds true for almost surely $(X,Y)$, then $G\equiv G'$. 
\end{proposition}
The proof of Proposition~\ref{prop:identifiability} is in Appendix~\ref{appendix:identifiability}. This result ensures the convergence of the MLE $\widehat{G}_n$ to the true mixing measure $G_*$ as long as the conditional density $g_{\widehat{G}_n}$ converges to $g_{G_*}$ for almost surely $(X,Y)$. Next, we characterize the density estimation rate under the Hellinger distance in the following proposition:

\begin{proposition}[Density Estimation Rate]
    \label{prop:density_estimation}
    With the MLE $\widehat{G}_n$ defined in equation~\eqref{eq:MLE}, the convergence rate of the density estimation $g_{\widehat{G}_n}$ to the true density $g_{G_*}$ is given by:
    \begin{align}
        \label{eq:density_rate}
        \mathbb{P}\Big(\bbE_X[h(g_{\widehat{G}_n}(\cdot|X),g_{G_*}(\cdot|X))]&>C_0\sqrt{\log(n)/n}\Big)\nonumber\\
        &\lesssim\exp(-c_0\log(n)),
    \end{align}
    where $C_0$ and $c_0$ are universal positive constants that depend only on $\Theta$.
\end{proposition}
The proof of Proposition~\ref{prop:density_estimation} is in Appendix~\ref{appendix:density_estimation}. It is evident from the bound in equation~\eqref{eq:density_rate} that the rate for estimating $g_{G_*}$ is parametric of order $\widetilde{\mathcal{O}}(n^{-1/2})$. Therefore, if we are able to construct a metric $\mathcal{D}$ among parameters that satisfies the inequality $\bbE_X[h(g_{G}(\cdot|X),g_{G_*}(\cdot|X))]\gtrsim\mathcal{D}(G,G_*)$ for any mixing measure $G\in\mathcal{O}_k(\Theta)$, then the MLE $\widehat{G}_n$ will also converge to the true mixing measure $G_*$ at the parametric rate of order $\widetilde{\mathcal{O}}(n^{-1/2})$. 

\textbf{Voronoi-based Loss:} As the MLE $\widehat{G}_n$ is constrained to have more components than its true counterpart, there exists a true component approximated by at least two components while others could be fitted by exactly one component. This possibly leads to different estimation rates among the true parameters, which requires us to design novel loss functions tailored to this observation. To this end, let us recall a notion of Voronoi cells, which was previously studied in \cite{manole22refined}. In particular, for any mixing measure $G\in\mathcal{O}_k(\Theta)$, a Voronoi cell of $G$ generated by $\theta^*_j:=(\beta^*_{1j},a^*_{j},b^*_{j})$ is defined as  $\mathcal{C}_j\equiv\mathcal{C}_j(G):=\{i\in[k]:\|\theta_i-\theta^*_j\|\leq\|\theta_i-\theta^*_{\ell}\|,\ \forall\ell\neq j\}$ for any $j\in[k_*]$, where $\theta_i:=(\beta_{1i},a_{i},b_{i})$. Here, the cardinality of each Voronoi cell $\mathcal{C}_j$ indicates the number of fitted components approximating the true component $\theta^*_j$. Based on these Voronoi cells, we define a generic Voronoi-based metric of order $r\geq 1$ between any two mixing measures, denoted by $\mathcal{D}_r(G,G_*)$, as follows:
\begin{align}
\label{eq:Voronoi_loss_form}
&\mathcal{D}_r(G,G_*):=\sum_{j=1}^{k_*}\left|\sum_{i\in\mathcal{C}_j}\exp(\beta_{0i})-\exp(\beta^*_{0j})\right|,\nonumber\\
&+\sum_{\substack{j:|\mathcal{C}_j|>1, \\ i\in\mathcal{C}_j}}\exp(\beta_{0i})\Big[\|\Delta\beta_{1ij}\|^r+\sum_{\ell=1}^{K-1}(|\Delta a_{ij\ell}|^r+\|\Delta b_{ij\ell}\|^r)\Big]\nonumber\\
&+\sum_{\substack{j:|\mathcal{C}_j|=1, \\ i\in\mathcal{C}_j}}\exp(\beta_{0i})\Big[\|\Delta\beta_{1ij}\|+\sum_{\ell=1}^{K-1}(|\Delta a_{ij\ell}|+\|\Delta b_{ij\ell}\|)\Big],
\end{align}
where $\Delta\beta_{1ij}:=\beta_{1i}-\beta^*_{1j}$, $\Delta a_{ij\ell}:=a_{i\ell}-a^*_{j\ell}$ and $\Delta b_{ij\ell}:=b_{i\ell}-b^*_{j\ell}$. The above Voronoi loss function allows us to capture precisely the distinct convergence behaviors of parameter estimations under the multinomial logistic MoE model with (modified) softmax gating function.

\section{Standard Softmax Gating Multinomial Logistic MoE}
\label{sec:standard}

In this section, we present the convergence analysis of parameter estimation in the standard softmax gating multinomial logistic MoE model.

Given the parametric density estimation rate in Proposition~\ref{prop:density_estimation}, we observe that if the Hellinger lower bound
\begin{align*}
    \bbE_X[h(g_{\widehat{G}_n}(\cdot|X),g_{G_*}(\cdot|X))]\gtrsim\mathcal{D}_r(\widehat{G}_n,G_*),
\end{align*}
holds true for some $r\geq 1$, then the convergence rate of the MLE $\widehat{G}_n$ to the true mixing measure $G_*$ would also be parametric of order $\widetilde{\mathcal{O}}(n^{-1/2})$. Therefore, our goal is to establish the above bound. Since the Hellinger distance is lower bounded by the Total Variation distance, i.e., $h\geq V$, it is sufficient to show that
\begin{align}
    \label{eq:TV_bound}
    \bbE_X[V(g_{\widehat{G}_n}(\cdot|X),g_{G_*}(\cdot|X))]\gtrsim\mathcal{D}_r(\widehat{G}_n,G_*).
\end{align}
Then, we consider the term $V(g_{\widehat{G}_n}(\cdot|X),g_{G_*}(\cdot|X))$ which involves the density discrepancy $g_{\widehat{G}_n}(Y|X)-g_{G_*}(Y|X)$. We aim to decompose this discrepancy into a linear combination of linearly independent terms using Taylor expansion so that when $g_{\widehat{G}_n}$ converges to $g_{G_*}$, the associated coefficients involving the parameter differences $\widehat{\beta}^n_{1i}-\beta^*_{1j}$, $\widehat{a}^n_i-a^*_j$ and $\widehat{b}^n_i-b^*_j$ also tend to zero, where $(\widehat{\beta}^n_{1i},\widehat{a}^n_i,\widehat{b}^n_i)$ are components of $\widehat{G}_n$. For that purpose, we first rewrite the previous density discrepancy in terms of
\begin{align*}
    u(Y|X;\widehat{\beta}^n_{1i},\widehat{a}^n_i,\widehat{b}^n_i)-u(Y|X;\beta^*_{1j},a^*_j,b^*_j),
\end{align*}
where $u(Y|X;\beta_{1i},a_i,b_i)=\exp((\beta_{1i})^{\top}X)f(Y|X;a_i,b_i)$. Next, we apply the Taylor expansion to the function $u(Y|X;\widehat{\beta}^n_{1i},\widehat{a}^n_i,\widehat{b}^n_i)$ about the point $(\beta^*_{1j},a^*_j,b^*_j)$. In this step, we notice that if there exists $j\in[k_*]$ such that $b^*_{j\ell}=\zerod$ for any $\ell\in[K-1]$, then the gating parameters interact with the expert parameters via the following PDEs:
\begin{align}
    &\frac{\partial u}{\partial\beta_{1j}}(Y=s|X;\beta^*_{1j},a^*_{j},b^*_{j})\nonumber\\
    \label{eq:PDE_repeat}
    &\hspace{1.5cm}=C_{a^*_j}\cdot\frac{\partial u}{\partial b_{js}}(Y=s|X;\beta^*_{1j},a^*_{j},b^*_{j}),
\end{align}
for any $s\in[K]$, where $C_{a^*_j}>0$ is some constant depending only on $a^*_j$. This interaction indicates that there are several linearly dependent derivative terms in the previous Taylor expansion, which induces a number of challenges in representing the density discrepancy $g_{\widehat{G}_n}(Y|X)-g_{G_*}(Y|X)$ as a linear combination of linearly independent terms. On the other hand, if for any $j\in[k_*]$, we can find an index $\ell\in[K-1]$ such that $b^*_{j\ell}\neq\zerod$, then the PDEs inequation~\eqref{eq:PDE_repeat} no longer hold true and the previous interaction does not occur, which facilitates the decomposition of the density discrepancy. 

For those reasons, we will characterize the parameter estimation rates under two following complement regimes of expert parameters $b^*_j$ in Section~\ref{sec:regime1} and Section~\ref{sec:regime2}, respectively:
\begin{itemize}
    \item \textbf{Regime 1:} For any $j\in[k_*]$, there exists an index $\ell\in[K-1]$ such that $b^*_{j\ell}\neq\zerod$;
    \item \textbf{Regime 2:} There exists an index $j\in[k_*]$ such that $b^*_{j\ell}=\zerod$ for any $\ell\in[K-1]$.
\end{itemize}

\subsection{Regime 1 of Expert Parameters} 
\label{sec:regime1}
Under this regime, it can be verified that the PDEs in equation~\eqref{eq:PDE_repeat} are no longer valid for any $j\in[k_*]$. Then, we provide in the following theorem the convergence rate of the MLE $\widehat{G}_n$ to its true counterpart $G_*$ under the Voronoi loss function $\mathcal{D}_2$ in equation~\eqref{eq:Voronoi_loss_form}.
\begin{theorem}[Parameter Estimation Rate]
    \label{theorem:standard_one}
    Suppose that the assumption of Regime 1 holds true, then we achieve the Hellinger lower bound
    \begin{align*}
        \bbE_X[h(g_{G}(\cdot|X),g_{G_*}(\cdot|X))]\gtrsim\mathcal{D}_2(G,G_*)
    \end{align*}
    for any mixing measure $G\in\mathcal{O}_k(\Theta)$. Therefore, combining this bound with Proposition~\ref{prop:density_estimation} leads to the following convergence rate of the MLE $\widehat{G}_n$:
    \begin{align*}
        \bbP\Big(\mathcal{D}_{2}(\widehat{G}_n,G_*)>C_1\sqrt{\log(n)/n}\Big)\lesssim \exp(-c_1\log(n)),
    \end{align*}
    where $C_1>0$ is a constant depending on $\Theta$ and $G_*$, while the constant $c_1>0$ depends only on $\Theta$.
\end{theorem}
Proof of Theorem~\ref{theorem:standard_one} is in Appendix~\ref{appendix:standard_one}. A few comments regarding this result are in order. 

(i) Theorem~\ref{theorem:standard_one} suggests that the MLE $\widehat{G}_n$ converges to the true mixing measure $G_*$ under the loss function $\mathcal{D}_2$ at the parametric rate of order $\widetilde{\mathcal{O}}(n^{-1/2})$. Based on the formulation of $\mathcal{D}_2$ in equation~\eqref{eq:Voronoi_loss_form}, it follows that exact-specified parameters $\beta^*_{1j},a^*_j,b^*_j$, whose Voronoi cells $\mathcal{C}_j$ have exactly one element, share the same estimation rate of order $\widetilde{\mathcal{O}}(n^{-1/2})$. 

(ii) On the other hand, for over-specified parameters $\beta^*_{1j},a^*_j,b^*_j$ whose Voronoi cells $\mathcal{C}_j$ have more than one element, the rates for estimating them are slower, standing at order $\widetilde{\mathcal{O}}(n^{-1/4})$. Nevertheless, these rates are substantially faster than their counterparts under the softmax gating Gaussian MoE model studied in \cite{nguyen2023demystifying}, which decreases monotonically with the cardinality of Voronoi cells.

\subsection{Regime 2 of Expert Parameters} 
\label{sec:regime2}
Under this regime, we can find an index $j\in[k_*]$ that satisfies $b^*_{j\ell}=0$ for any $\ell\in[K-1]$. By simple calculations, we can validate that the PDEs in equation~\eqref{eq:PDE_repeat} hold true. This result leads to so many linearly dependent terms among the derivatives of the function $u$ w.r.t its parameters that the density discrepancy $g_{\widehat{G}_n}(Y|X)-g_{G_*}(Y|X)$ cannot be decomposed into a linear combination of linearly independent elements as our expectation. As a consequence, we demonstrate in the following proposition that the Total Variation lower bound in equation~\eqref{eq:TV_bound} no longer holds under Regime 2. 
\begin{proposition}
    \label{prop:TV_prop}
    Suppose that the assumption of Regime 2 is satisfied, then we obtain that
    \begin{align*}
        \inf_{\substack{G\in\mathcal{O}_k(\Theta), \\ \mathcal{D}_r(G,G_*)\leq\varepsilon}}\bbE_X[V(g_{G}(\cdot|X),g_{G_*}(\cdot|X))]/\mathcal{D}_r(G,G_*)\to0,
    \end{align*}
    as $\varepsilon\to0$, for any $r\geq 1$.
\end{proposition}
Proof of Proposition~\ref{prop:TV_prop} is in Appendix~\ref{appendix:TV_prop}. Based on the above result, we establish a minimax lower bound for estimating the true mixing measure $G_*$ in Theorem~\ref{theorem:standard_two}.
\begin{theorem}[Minimax Lower Bound]
    \label{theorem:standard_two}
    Under Regime 2, the following minimax lower bound of estimating $G_*$ holds true for any $r\geq 1$:
    \begin{align*}
        \inf_{\overline{G}_n\in\mathcal{O}_{k}(\Theta)}\sup_{G\in\mathcal{O}_{k}(\Theta)\setminus\mathcal{O}_{k_*-1}(\Theta)}\bbE_{g_{G}}[\mathcal{D}_r(\overline{G}_n,G)]\gtrsim n^{-1/2}.
    \end{align*}
    Here, $\bbE_{g_G}$ represents the expectation taken with respect to the product measure with mixture density $g^n_{G}$.
\end{theorem}
Proof of Theorem~\ref{theorem:standard_two} can be found in Appendix~\ref{appendix:standard_two}. This result together with the formulation of $\mathcal{D}_r$ in equation~\eqref{eq:Voronoi_loss_form} indicates that the rates for estimating exact-specified parameters $\beta^*_{1j},a^*_j,b^*_j$ are not better than $\widetilde{\mathcal{O}}(n^{-1/2})$, while those for their over-specified counterparts are even slower than any polynomial rates.
\subsection{Proof Sketches}
\label{sec:proof_sketch}
In this section, we provide proof sketches for both Theorem~\ref{theorem:standard_one} and Theorem~\ref{theorem:standard_two}. 
\subsubsection{Proof Sketch of Theorem~\ref{theorem:standard_one}}
Due to the inequality $h\geq V$, it suffices to show the Total Variation lower bound $\bbE_X[V(g_{G}(\cdot|X),g_{G_*}(\cdot|X)]\gtrsim \mathcal{D}_2(G,G_*)$ for any $G\in\mathcal{O}_k(\Theta)$. We divide this problem into two parts as follows:

\textbf{Local structure.} In this part, we need to show that 
\begin{align*}
    \lim_{\varepsilon\to 0}\inf_{\substack{G\in\mathcal{O}_k(\Theta):\\ \mathcal{D}_2(G,G_*)\leq\varepsilon}}\frac{\bbE_X[V(g_{G}(\cdot|X),g_{G_*}(\cdot|X))]}{\mathcal{D}_{2}(G,G_*)}>0.
\end{align*}
Assume by contrary that there exists $G_n\in\mathcal{O}_k(\Theta)$ such that $\bbE_X[V(g_{G_n}(\cdot|X),g_{G_*}(\cdot|X))]/\mathcal{D}_2(G_n,G_*)$ and $\mathcal{D}_2(G_n,G_*)$ both converge to zero. 

\textbf{Step 1.} In this step, we use Taylor expansions to decompose the quantity $T_n(s):=\Big[\sum_{j=1}^{k_*}\exp((\beta^*_{1j})^{\top}X+\beta^*_{0j})\Big]\cdot[g_{G_n}(Y=s|X)-g_{G_*}(Y=s|X)]$ into a linear combination of linearly independent terms as
\begin{align*}
    T_n(s)=\sum_{j=1}^{k_*}\omega_jE_j(Y=s|X)+R(X,Y),
\end{align*}
where $R(X,Y)$ is a Taylor remainder such that $R(X,Y)/\mathcal{D}_2(G_n,G_*)\to0$.

\textbf{Step 2.} In this step, we show by contradiction that not all the ratios $\omega_j/\mathcal{D}_2(G_n,G_*)$ tend to zero as $n\to\infty$. Assume that all of them converge to zero, then we deduce that $1=\mathcal{D}_2(G_n,G_*)/\mathcal{D}_2(G_n,G_*)\to0$, which is a contradiction. Therefore, at least one among the ratios $w_j/\mathcal{D}_2(G_n,G_*)$ does not vanish.

\textbf{Step 3.} Since $\bbE_X[V(g_{G_n}(\cdot|X),g_{G_*}(\cdot|X))]/\mathcal{D}_2(G_n,G_*)$ converges to zero, then by applying the Fatou's lemma, we deduce that $T_n(s)/\mathcal{D}_2(G_n,G_*)\to0$ as $n\to\infty$. Note that $T_n(s)$ is written as a linear combination of linearly independent terms, thus, all the associated coefficients in that combination go to zero, which contradicts the result in Step 2. Hence, the proof of the local structure is completed.

\textbf{Global Structure.} In this part, we demonstrate that
\begin{align*}
    \inf_{\substack{G\in\mathcal{O}_k(\Theta),\\ \mathcal{D}_2(G,G_*)>\varepsilon}}\frac{\bbE_X[V(g_{G}(\cdot|X),g_{G_*}(\cdot|X))]}{\mathcal{D}_{2}(G,G_*)}>0.
\end{align*}
Assume that this claim is not true. Then, we can find a mixing measure $G'\in\mathcal{O}_k(\Theta)$ such that $\bbE_X[V(g_{G}(\cdot|X),g_{G_*}(\cdot|X))]=0$. By the Fatou's lemma, we obtain that $g_{G'}(Y=s|X)=g_{G_*}(Y=s|X)$ for any $s\in[K]$ for almost surely $X$. It follows from Proposition~\ref{prop:density_estimation} that $G'\equiv G$. This result implies that $\mathcal{D}_2(G',G_*)=0$, which contradicts the constraint $\mathcal{D}_2(G',G_*)>\varepsilon>0$. Hence, the proof sketch is completed.

\subsubsection{Proof Sketch of Theorem~\ref{theorem:standard_two}}
Let $M_1$ be some fixed positive constant. Following from the result of Proposition~\ref{prop:TV_prop}, for any suffciently small $\varepsilon>0$, we can seek a mixing measure $G'_*\in\mathcal{O}_k(\Theta)$ such that $\mathcal{D}_r(G'_*,G_*)=2\varepsilon$ and $\bbE_X[V(g_{G'_*}(\cdot|X),g_{G_*}(\cdot|X))]\leq M_1\varepsilon$. Note that for any sequence of mixing measures $\overline{G}_n\in\mathcal{O}_k(\Theta)$, we have
\begin{align*}
    2\max_{G\in\{G_*,G'_*\}}&\bbE_{g_G}[\mathcal{D}_r(\overline{G}_n,G)]\\
    &\geq \bbE_{g_{G_*}}[\mathcal{D}_r(\overline{G}_n,G_*)]+\bbE_{g_{G'_*}}[\mathcal{D}_r(\overline{G}_n,G'_*)].
\end{align*}
Moreover, since $\mathcal{D}_r$ satisfies the weak triangle inequality, there exists some constant $M_2>0$ such that
\begin{align*}
    \mathcal{D}_r(\overline{G}_n,G_*)+\mathcal{D}_r(\overline{G}_n,G'_*)\geq M_2\mathcal{D}_r(G_*,G'_*)=2M_2\varepsilon.
\end{align*}
Then, by leveraging the Le Cam's minimax lower bound approach \cite{yu97lecam}, we get that
\begin{align*}
    \max_{G\in\{G_*,G'_*\}}&\bbE_{g_G}[\mathcal{D}_r(\overline{G}_n,G)]\\
    &\geq M_2\varepsilon(1-\bbE_X[V(g^n_{G_*}(\cdot|X),g^n_{G'_*}(\cdot|X))])\\
    &\geq M_2\varepsilon\Big[1-\sqrt{1-(1-M_1^2\varepsilon^2)^n}\Big].
\end{align*}
By setting $\varepsilon=n^{-1/2}/M_1$, we obtain that 
\begin{align*}
    \max_{G\in\{G_*,G'_*\}}\bbE_{g_G}[\mathcal{D}_r(\overline{G}_n,G)]\gtrsim n^{-1/2},
\end{align*}
for any sequence $\overline{G}_n\in\mathcal{O}_k(\Theta)$. Furthermore, since $\{G_*,G'_*\}$ is a subset of $\mathcal{O}_k(\Theta)\setminus\mathcal{O}_{k_*-1}(\Theta)$, we reach the conclusion of Theorem~\ref{theorem:standard_two}.

\section{Modified Softmax Gating Multinomial Logistic MoE}
\label{sec:modified}
In this section, we propose a novel class of modified softmax gating functions to resolve the interaction between gating parameters and expert parameters via the PDE~\eqref{eq:PDE_repeat} which leads to significantly slow parameter estimation rates. Then, we also capture the convergence rates of parameter estimation under the multinomial logistic MoE model with that those gating functions.

First of all, let us introduce the formulation of the modified softmax gating multinomial logistic MoE model.

\textbf{Problem Setup.} Suppose that the i.i.d samples $(X_1,Y_1),(X_2,Y_2),\ldots,(X_n,Y_n)\in\mathcal{X}\times[K]$ are drawn from the multinomial logistic MoE model of order $k_*$ with modified softmax gating function, whose conditional density function $\widetilde{g}_{G_*}(Y=s|X)$ is given by
\begin{align}
    &\sum_{i=1}^{k_*}~\softmax((\beta^*_{1i})^{\top}M(X)+\beta^*_{0i})\times f(Y=s|X;a^*_{i},b^*_{i})\nonumber\\
    &:=\sum_{i=1}^{k_*}~\frac{\exp((\beta^*_{1i})^{\top}M(X)+\beta^*_{0i})}{\sum_{j=1}^{k_*}\exp((\boj)^{\top}M(X)+\bzj)}\nonumber\\
    \label{eq:modified_conditional_density}
    &\hspace{2.2cm}\times\frac{\exp(a^*_{is}+(b^*_{is})^{\top}X)}{\sum_{\ell=1}^{K}\exp(a^*_{i\ell}+(b^*_{i\ell})^{\top}X)},
\end{align}
for any $s\in[K]$, where $M$ is a function defined in Definition~\ref{def:modified_function}. Here, we reuse all the assumptions imposed on the model in equation~\eqref{eq:conditional_density} unless stating otherwise. In the above model, we would transform the covariate $X$ using the function $M$ first rather than routing it directly to the softmax gating function as in the model~\eqref{eq:conditional_density}. This transformation step allows us to overcome the interaction between gating parameters and expert parameters emerging from Section~\ref{sec:regime2}, which we will discuss below. 

\begin{definition}[Modified Function]
\label{def:modified_function}
Let $M:\mathcal{X}\to\mathbb{R}^d$ be a bounded function such that the following set is linearly independent for almost surely $X$:
\begin{align}
    \label{eq:linear_independent_set}
    \Big\{X^p[M(X)]^{q}: p,q\in\mathbb{N}^d, \ 0\leq|p|+|q|\leq 2\Big\}.
\end{align}
\end{definition}
To understand the above condition better, we provide below both intuitive and technical explanations for it.

\textbf{Intuitively,} Under the multinomial logistic MoE model with the standard softmax gate, we observe an interaction among the gating parameters $\beta_1$ and the expert parameters $b$ when a fraction of expert parameters vanish (see equation~\eqref{eq:PDE}. This interaction mainly accounts for the slow parameter estimation rates, which could be as slow as $\mathcal{O}(1/\log(n))$. We realize that the previous interaction occurs as parameters $\beta_1$ and $b$ are both associated with the input $X$ in the condition density in equation~\eqref{eq:conditional_density}. To address this issue, we propose transforming the input $X$ in the softmax gate by the function $M$. Then, in the modified conditional density in equation~\eqref{eq:modified_conditional_density}, $\beta_1$ is associated with $M(X)$, while $b$ is still with $X$. However, to eliminate the parameter interaction completely, we need to impose an assumption of linear independence between the input $X$ and its transformation $M(X)$ in Definition~\ref{def:modified_function}. 

\textbf{Technically,} the set in equation~\eqref{eq:linear_independent_set} is assumed to be linearly independent for almost surely $X$ to guarantee there does not exist any constant $C_{a^*_i}$ depending only $a^*_i$ such that the PDEs
\begin{align*}
    &\frac{\partial u}{\partial\beta_{1j}}(Y=s|X;\beta^*_{1j},a^*_{j},b^*_{j})\\
    &\hspace{1.5cm}=C_{a^*_j}\cdot\frac{\partial u}{\partial b_{js}}(Y=s|X;\beta^*_{1j},a^*_{j},b^*_{j}),
\end{align*}
for $s\in[K]$, do not hold under any setting of the expert parameters $b^*_i$. As a result, the interaction between gating parameters and expert parameters mentioned in the Regime 2 does not appear in the modified softmax gating multinomial logistic MoE model. Next, we provide below a few valid instances of the function $M$.

\textbf{Example.} It can verified that all the following element-wise functions satisfy the condition in Definition~\ref{def:modified_function}: $M(X)=\tanh(X)$, $M(X)=\cos(X)$ and $M(X)=X^m$ for $m\geq 3$. Additionally, $M$ could also be a normalization function, i.e. $M(X)=\frac{X}{||X||}$. The practical problem of selecting important experts can benefit from this choice of function $M$. In particular, in many applications of MoE models, there are some experts which do not play an essential role in learning, and even become redundant. In such scenario, we would like the gate to put more weights on the important experts. However, if the input magnitude is huge, then the weight distribution will become uniform, which is undesirable. Therefore, the input normalization function $M$, which helps remain the input magnitude of one, is beneficial in this case. Regarding the parameter estimation problem, since the function $M$ helps remove the parameter interaction, the parameter estimation rates are ensured to be of polynomial orders under any parameter settings (see Table~\ref{table:parameter_rates}). Another potential option is $M(X)=sigmoid(X)$, which allows the magnitude of the input to vary between 0 and 1.

\textbf{Maximum Likelihood Estimation.} Due to the modification of the conditional density in equation~\eqref{eq:modified_conditional_density}, the formulation of MLE of the true mixing measure $G_*$ under the modified softmax gating multinomial logistic MoE model also changes as follows: 
\begin{align}
    \label{eq:modified_MLE}
    \widetilde{G}_n\in\argmax_{G\in\mathcal{O}_k(\Theta)}\sum_{i=1}^{n}\log(\widetilde{g}_{G}(Y_i|X_i)).
\end{align}

Next, we validate the identifiability of the modified softmax gating multinomial logistic MoE model in the following proposition.
\begin{proposition}[Identifiability]
    \label{prop:modified_identifiability}
    Assume that $G$ and $G'$ are two mixing measures in $\mathcal{O}_k(\Theta)$. If $\widetilde{g}_{G}(Y|X)=\widetilde{g}_{G'}(Y|X)$ holds true for almost surely $(X,Y)$, then we obtain that $G\equiv G'$. 
\end{proposition}
Proof of Proposition~\ref{prop:modified_identifiability} is in Appendix~\ref{appendix:proof_density_estimation}. This result points out that the modified softmax gating multinomial logistic MoE model is still identifiable. In other words, if the density estimation $\widetilde{g}_{\widetilde{G}_n}$ converges to the true density $\widetilde{g}_{G_*}$, then the MLE $\widetilde{G}_n$ also approaches its true counterpart $G_*$.

Now, we are ready to derive the convergence rate of density estimation $\widetilde{g}_{\widetilde{G}_n}$ to the true density $\widetilde{g}_{G_*}$ under the Hellinger distance.

\begin{proposition}[Density Estimation Rate]
    \label{prop:modified_density_estimation}
    With the MLE $\widetilde{G}_n$ defined in equation~\eqref{eq:modified_MLE}, the density estimation $\widetilde{g}_{\widetilde{G}_n}$ converges to the true density $\widetilde{g}_{G_*}$ at the following rate:
    \begin{align}
        \label{eq:modified_density_rate}
        \mathbb{P}\Big(\bbE_X[h(\widetilde{g}_{\widetilde{G}_n}(\cdot|X),\widetilde{g}_{G_*}&(\cdot|X))]>C_2\sqrt{\log(n)/n}\Big)\nonumber\\
        &\lesssim\exp(-c_2\log(n)),
    \end{align}
    where $C_2$ and $c_2$ are universal positive constants that depend only on $\Theta$.
\end{proposition}

Proof of Proposition~\ref{prop:modified_density_estimation} is in Appendix~\ref{appendix:modified_density_estimation}. It follows from the above proposition that the density estimation rate under the modified softmax gating multinomial logistic MoE model is parametric of order $\widetilde{\mathcal{O}}(n^{-1/2})$. This result matches the rate for estimating the true density under the multinomial logistic MoE model with standard softmax gating function in Proposition~\ref{prop:density_estimation}. Based on this observation, we then capture the convergence behavior of the MLE $\widetilde{G}_n$ in the following theorem.

\begin{theorem}[Parameter Estimation Rate]
    \label{theorem:modified_softmax}
    The following Hellinger lower bound holds true for any mixing measure $G\in\mathcal{O}_k(\Theta)$:
    \begin{align*}
        \bbE_X[h(\widetilde{g}_{G}(\cdot|X),\widetilde{g}_{G_*}(\cdot|X))]\gtrsim\mathcal{D}_2(G,G_*).
    \end{align*}
    This lower bound together with Proposition~\ref{prop:modified_density_estimation} imply that there exists a constant $C_3>0$ depending on $\Theta$ and $G_*$ such that
    \begin{align*}
        \bbP(\mathcal{D}_{2}(\widetilde{G}_n,G_*)>C_3\sqrt{\log(n)/n})\lesssim\exp(-c_3\log(n)),
    \end{align*}
    where $c_3>0$ is a constant that depends only on $\Theta$.
\end{theorem}
Proof of Theorem~\ref{theorem:modified_softmax} is deferred to Appendix~\ref{appendix:modified_softmax}. This theorem reveals that the MLE $\widetilde{G}_n$ converges to the true mixing measure $G_*$ under the Voronoi loss function $\mathcal{D}_2$ at a rate of order $\widetilde{\mathcal{O}}(n^{-1/2})$. From the definition of $\mathcal{D}_2$ in equation~\eqref{eq:Voronoi_loss_form}, we deduce that the exact-specified parameters $\beta^*_{1j},a^*_j,b^*_j$, which are fitted by exactly one component, enjoy the same estimation rate of $\widetilde{\mathcal{O}}(n^{-1/2})$, whereas those for their over-specified counterparts, i.e. those fitted by at least two components, are of order $\widetilde{\mathcal{O}}(n^{-1/4})$. It is worth emphasizing that these rates remain stable regardless of any values of the expert parameters $b^*_i$. This highlights the benefits of modified softmax gating functions over the standard softmax gating in the multinomial logistic MoE model.

In summary, replacing the standard softmax gating function with its modified versions in the multinomial logistic MoE model remains the identifiability of that model and the parametric density estimation rate as well as ensures the stability of parameter estimation rates under any settings of the expert parameters (see Table~\ref{table:parameter_rates}). As a consequence, we can conclude that the modified softmax gating functions outperform their standard counterpart in the parameter estimation problem under the multinomial logistic MoE model. 

\vspace{0.5em}

\section{Discussion}
\vspace{0.5em}

\label{sec:conclusion}
In this paper, we investigate the convergence behavior of maximum likelihood estimation in the softmax gating multinomial logistic mixture of experts under the over-specified settings. For that purpose, we design a novel generic Voronoi loss function, and then discover that the rates for estimating the true density and exact-specified true parameters are all parametric on the sample size, while those for over-specified true parameters are slightly slower. However, when part of the expert parameters vanishes, the estimation rates for the over-specified true parameters experience a surprisingly slow rates due to an intrinsic interaction between the parameters of gating and expert functions. To tackle this issue, we propose a novel class of modified softmax gating functions that not only keep the identifiability of the multinomial logistic mixture of experts and the parametric density estimation rate unchanged, but also stabilize the parameter estimation rates irrespective of the collapse of expert parameters. This highlights the advantages of using modified softmax gating functions over the standard softmax gating function in the parameter estimation problem of the multinomial logistic MoE model.


\textbf{Technical novelty.} Compared to previous works, our paper is technically novel in terms of the following aspects:

\textbf{(1) Covariate-dependent gate:} We consider softmax gate whose value depends on the covariate $X$, while \cite{ho2022gaussian} use covariate-free weights, which are significantly simpler. Thus, in Step 1 of our proofs of Theorems~\ref{theorem:standard_one} and \ref{theorem:modified_softmax}, if we apply Taylor expansions directly to the density discrepancy $g_{G_n}(Y|X)-g_{G_*}(Y|X)$ as in \cite{ho2022gaussian}, we cannot represent that discrepancy as a combination of elements from some linearly independent set, which is a key step. Therefore, we have to take the product of the softmax's denominator and the density discrepancy, denoted by $T_n(s):=\Big[\sum_{j=1}^{k_*}\exp((\beta^*_{1j})^{\top}X+\beta^*_{0j})\Big]\cdot[g_{G_n}(Y=s|X)-g_{G_*}(Y=s|X)]$. Then, we decompose $T_n(s)$ such that it includes two functions $\exp((\beta_{1i}^n)^{\top}X)f(Y=s|X;a_i^n,b_i^n)$ and $\exp((\beta_{1i}^n)^{\top}X)g_{G_n}(Y=s|X)$ (see equation~\eqref{eq:Tn_decomposition}). Then, we need to apply two Taylor expansions to those functions rather than only one as in \cite{ho2022gaussian}. 

\textbf{(2) Minimax lower bound:} A key step to derive polynomial rates for estimating parameters (even in extreme cases) in \cite{nguyen2023demystifying,nguyen2024statistical} is to establish Total Variation lower bounds $\mathbb{E}_{X}[V(g_{G}(\cdot|X),g_{G_*}(\cdot|X))]\gtrsim\mathcal{D}(G,G_*)$,
where $\mathcal{D}$ is a loss function. However, we show in Proposition~\ref{prop:TV_prop} that such bound does not hold true under the Regime 2 due to the PDEs in equation~\eqref{eq:PDE}. Thus, we provide a minimax lower bound in Theorem~\ref{theorem:standard_two} to show that the parameter estimation rates are slower than any polynomial rates, and therefore, could be as slow as $\mathcal{O}(1/\log(n))$. As far as we are concerned, this phenomenon has never been observed in previous works. 

\textbf{(3) Non-trivial solution for rate improvement:} \cite{ho2022gaussian,nguyen2023demystifying,nguyen2024statistical} characterized the slowest rates of the same form $\mathcal{O}(n^{-1/2\bar{r}})$, for some $\bar{r}\geq 4$, without proposing any solutions to improve them. By contrast, although the slowest rate in our work is even slower than any polynomial rates, we provide a non-trivial modified gating function in Section 4 to alleviate that issue. More importantly, since that modified gate generally helps address interactions among gating and expert parameters, it can be applied to \cite{ho2022gaussian,nguyen2023demystifying,nguyen2024statistical} (up to some changes of the conditions of function $M(\cdot)$ in Definition~\ref{def:modified_function}) for rate improvement. As shown in Table~\ref{table:parameter_rates}, the modified softmax gate enhances the estimation rates from $\mathcal{O}(n^{-1/2\bar{r}})$ to at least $\mathcal{O}(n^{-1/4})$ under any parameter settings. To the best of our knowledge, such flexible and effective solution had never been proposed in the literature.

\vspace{0.5em}
\textbf{Limitation and future directions.} In this paper, we consider a well-specified setting when the data are assumed to be sampled from a softmax gating multinomial logistic MoE. However, in practice, the data are not necessarily generated from that model, which we refer to as the misspecified setting. Under the misspecified setting, the MLE converges to a mixing measure $\widetilde{G} \in \argmin_{G \in \mathcal{O}_{k}(\Theta)} \text{KL}(P(Y|X) || g_{G}(Y|X))$, where $P(Y|X)$ is the true conditional distribution of $Y$ given $X$ and KL stands for the Kullback-Leibler divergence. As the space $\mathcal{O}_{k}(\Theta)$ is non-convex, the existence of $\widetilde{G}$ is not unique. Furthermore, the current analysis of the MLE under the misspecified setting of statistical models is mostly conducted when the function space is convex \cite{Vandegeer-2000}. Thus, it is necessary to develop new technical tools to establish the convergence rate of the MLE under the non-convex misspecified setting. Since this is beyond the scope of our work, we leave it for future development. Another potential direction is to comprehend the effects of the temperature parameter \cite{nie2022evomoe,nguyen2024temperature}, which controls the softmax weight distribution and the
sparsity of the MoE, on the convergence of parameter estimation under the softmax gating multinomial logistic MoE. This direction has remained unexplored in the literature.

\section*{Acknowledgements}
NH acknowledges support from the NSF IFML 2019844 and the NSF AI Institute for Foundations of Machine Learning. 
\section*{Impact Statement}

This paper presents work whose goal is to advance the field of Machine Learning. 
Given the theoretical nature of the paper, we believe that there are no potential societal consequences of our work.

\bibliography{icml_references}
\bibliographystyle{icml2024}

\clearpage
\onecolumn
\appendix

\centering
\textbf{\Large{Supplementary Material for 
``A General Theory for Softmax Gating \\ \vspace{0.1cm} Multinomial Logistic Mixture of Experts''}}

\justifying
\setlength{\parindent}{0pt}

In this supplementary material, we first present rigorous proofs for results regarding the convergence rates of parameter estimation under the (modified) softmax gating multinomial logistic mixture of experts in Appendix~\ref{appendix:parameter_estimation}, while those for density estimation are then provided in Appendix~\ref{appendix:general_density_estimation}. Next, we theoretically verify the identifiability of the proposed models in Appendix~\ref{appendix:general_identifiability}.
Lastly, we conduct several experiments in Appendix~\ref{appendix:simulation} to empirically validate our theoretical results.

\section{Proof for Convergence Rates of Parameter Estimation}
\label{appendix:parameter_estimation}
In this appendix, we provide proofs for Theorem~\ref{theorem:standard_one} and Theorem~\ref{theorem:standard_two} in Appendix~\ref{appendix:standard_one} and Appendix~\ref{appendix:standard_two}, respectively. Meanwhile, the proof of Theorem~\ref{theorem:modified_softmax} can be found in Appendix~\ref{appendix:modified_softmax}.
\subsection{Proof of Theorem~\ref{theorem:standard_one}}
\label{appendix:standard_one}

To reach the conclusion in Theorem~\ref{theorem:standard_one}, we need to show that
\begin{align}
\label{eq:original_inequality_one}
\inf_{G\in\mathcal{O}_{k}(\Theta)}\bbE_X[V(g_{G}(\cdot|X),g_{G_*}(\cdot|X))]/\mathcal{D}_{2}(G,G_*)>0.
\end{align}
\textbf{Local Structure}: Firstly, we prove by contradiction the local structure of inequality~\eqref{eq:original_inequality_one}, which is given by
\begin{align}
    \label{eq:local_structure_one}
    \lim_{\varepsilon\to 0}\inf_{\substack{G\in\mathcal{O}_k(\Theta),\\ \mathcal{D}_2(G,G_*)\leq\varepsilon}}\bbE_X[V(g_{G}(\cdot|X),g_{G_*}(\cdot|X))]/\mathcal{D}_{2}(G,G_*)>0.
\end{align}
Assume that this claim does not hold true, then there exists a sequence of mixing measures $G_n:=\sum_{i=1}^{k_n}\exp(\bzin)\delta_{(\boin,a^n_{i},b^n_{i})}\in\mathcal{O}_k(\Theta)$ such that both $\bbE_X[V(g_{G_n}(\cdot|X),g_{G_*}(\cdot|X))]/\mathcal{D}_2(G_n,G_*)$ and $\mathcal{D}_2(G_n,G_*)$ approach zero as $n$ tends to infinity. In addition, we define the set of Voronoi cells generated by the support of $G_*$ for this sequence as follows:
\begin{align*}
    \mathcal{C}^n_j:=\{i\in[k_n]:\|\theta^n_{i}-\theta^*_{j}\|\leq\|\theta^n_{i}-\theta^*_{\ell}\|,\quad \forall\ell\neq j\},
\end{align*}
where $\theta^n_i:=(\boin,a^n_{i},b^n_{i})$ and $\theta^*_j:=(\boj,a^*_{i},b^*_{i})$ for any $j\in[k_*]$. Recall that $\beta^*_{1k_*}=\zerod$ is known, thus we set $\beta^n_{1i}=\zerod$ for any $i\in\mathcal{C}_{k_*}$. Similarly, note that $(a^*_{jK},b^*_{jK})=(0,\zerod)$ are known for any $j\in[k_*]$, we also let $(a^n_{iK},b^n_{iK})=(0,\zerod)$ for any $i\in k_n$. As our proof argument is asymptotic, we can assume without loss of generality (WLOG) that the above Voronoi cells are independent of $n$, that is, $\mathcal{C}_j=\mathcal{C}^n_j$ for any $j\in[k_*]$. Next, since $\mathcal{D}_2(G_n,G_*)\to 0$ as $n\to\infty$, it follows from the formulation of $\mathcal{D}_2$ in equation~\eqref{eq:Voronoi_loss_form} that $\beta^n_{1i}\to\boj$, $(a^n_{i\ell},b^n_{i\ell})\to(a^*_{j\ell},b^*_{j\ell})$ and $\sum_{i\in\mathcal{C}^n_j}\exp(\beta^n_{0i})\to\exp(\beta^*_{0j})$ for any $\ell\in[K]$, $i\in\mathcal{C}_j$ and $j\in[k_*]$ when $n\to\infty$. Subsequently, we divide our remaining arguments into three steps as below.

\textbf{Step 1.} In this step, we use Taylor expansions to decompose the following quantity:
\begin{align*}
    T_n(s)&:=\Big[\sum_{j=1}^{k_*}\exp((\boj)^{\top}X+\bzj)\Big]\cdot\Big[g_{G_n}(Y=s|X)-g_{G_*}(Y=s|X)\Big].
\end{align*}
Denote $u(Y=s|X;\beta_{1i},a_i,b_i):=\exp(\beta_{1i}^{\top}X)\cdot f(Y|X;a_{is},b_{is})$ and $v(Y=s|X;\beta_{1i})=\exp(\beta_{1i}^{\top}X)g_{G_n}(Y=s|X)$ for any $s
\in[K]$, we have
\begin{align}\label{eq:Tn_decomposition}
    T_n(s)&=\sum_{j=1}^{k_*}\sum_{i\in\mathcal{C}_j}\exp(\bzin)\Big[u(Y=s|X;\boin,\ain,\bin)-u(Y=s|X;\boj,\aj,\bj)\Big]\nonumber\\
    &-\sum_{j=1}^{k_*}\sum_{i\in\mathcal{C}_j}\exp(\bzin)\Big[v(Y=s|X;\boin)-v(Y=s|X;\boj)\Big]\nonumber\\
    &+\sum_{j=1}^{k_*}\Big(\sum_{i\in\mathcal{C}_j}\exp(\bzin)-\exp(\bzj)\Big)\Big[u(Y=s|X;\boj,\aj,\bj)-v(Y=s|X;\boj)\Big]\nonumber\\
    &:=A_n-B_n+E_n.
\end{align}
Given the above formulations of $A_n$ and $B_n$, we continue to decompose them into two smaller terms based on the cardinality of Voronoi cells $\mathcal{C}_j$. In particular, 
\begin{align*}
    A_n&=\sum_{j:|\mathcal{C}_j|=1}\sum_{i\in\mathcal{C}_j}\exp(\bzin)\Big[u(Y=s|X;\boin,\ain,\bin)-u(Y=s|X;\boj,\aj,\bj)\Big]\\
    &+\sum_{j:|\mathcal{C}_j|>1}\sum_{i\in\mathcal{C}_j}\exp(\bzin)\Big[u(Y=s|X;\boin,\ain,\bin)-u(Y=s|X;\boj,\aj,\bj)\Big]\\
    &=A_{n,1}+A_{n,2}.
\end{align*}
Next, let us denote $h_{\ell}(X,a_{\ell},b_{\ell}):=a_{\ell}+b_{\ell}^{\top}X$ for any $\ell\in[K]$, then 
\begin{align*}
    f(Y=s|X;a_{i},b_{i})=\dfrac{\exp(a_{is}+b_{is}^{\top}X)}{\sum_{\ell=1}^K\exp(a_{i\ell}+b_{i\ell}^{\top}X)}=\dfrac{\exp(h_s(X,a_{is},b_{is}))}{\sum_{\ell=1}^{K}\exp(h_{\ell}(X,a_{i\ell},b_{i\ell}))}.
\end{align*}
By means of the first-order Taylor expansion, $A_{n,1}$ can be represented as
\begin{align*}
    A_{n,1}&=\sum_{j:|\mathcal{C}_j|=1}\sum_{i\in\mathcal{C}_j}\exp(\bzin)\sum_{|\alpha|=1}\frac{1}{\alpha!}(\dboijn)\prod_{\ell=1}^{K-1}(\daijnl)^{\alpha_{2\ell}}(\dbijnl)^{\alpha_{3\ell}}\frac{\partial^{|\alpha|}u(Y=s|X;\boj,\aj,\bj)}{\partial\beta_{1j}^{\alpha_1}\prod_{\ell=1}^{K-1}\partial a^{\alpha_{2\ell}}_{j\ell}\partial b^{\alpha_{3\ell}}_{j\ell}}+R_1(X,Y).
\end{align*}
Here, $\alpha:=(\alpha_1,\alpha_{21},\ldots,\alpha_{2(K-1)},\alpha_{31},\ldots,\alpha_{3(K-1)})$, where $\alpha_1,\alpha_{3\ell}\in\mathbb{N}^d$ and $\alpha_{2\ell}\in\mathbb{N}$ for any $\ell\in[K-1]$. Additionally, $R_1(X,Y)$ is a Taylor remainder such that $R_1(X,Y)/\mathcal{D}_2(G_n,G_*)\to0$ as $n\to\infty$. From the formulation of $u$, we have 
\begin{align*}
    A_{n,1}&=\sum_{j:|\mathcal{C}_j|=1}\sum_{i\in\mathcal{C}_j}\exp(\bzin)\sum_{|\alpha|=1}\frac{1}{\alpha!}(\dboijn)\prod_{\ell=1}^{K-1}(\daijnl)^{\alpha_{2\ell}}(\dbijnl)^{\alpha_{3\ell}}\\
    &\times X^{\alpha_1+\sum_{\ell=1}^{K-1}\alpha_{3\ell}}\exp((\boj)^{\top}X)\frac{\partial^{\sum_{\ell=1}^{K-1}(\alpha_{2\ell}+|\alpha_{3\ell}|)} f}{\partial h_1^{\alpha_{21}+|\alpha_{31}|}\ldots\partial h_{K-1}^{\alpha_{2(K-1)}+|\alpha_{3(K-1)}|}}(Y=s|X;a^*_j,b^*_j) + R_1(X,Y).
\end{align*}
Let $q_1=\alpha_1+\sum_{\ell=1}^{K-1}\alpha_{3\ell}\in\mathbb{N}^d$, $q_2=(q_{2\ell})_{\ell=1}^{K-1}:=(\alpha_{2\ell}+|\alpha_{3\ell}|)_{\ell=1}^{K-1}\in\mathbb{N}^{K-1}$ and 
\begin{align}
    \label{eq:set_I}
    \mathcal{I}_{q_1,q_2}:=\left\{\alpha=(\alpha_1,\alpha_{21},\ldots,\alpha_{2(K-1)},\alpha_{31},\ldots,\alpha_{3(K-1)}):\alpha_1+\sum_{\ell=1}^{K-1}\alpha_{3\ell}=q_{1}, \ (\alpha_{2\ell}+|\alpha_{3\ell}|)_{\ell=1}^{K-1}=q_{2}\right\},
\end{align}
we can rewrite $A_{n,1}$ as 
\begin{align*}
    A_{n,1}&=\sum_{j:|\mathcal{C}_j|=1}\sum_{|q_1|+|q_2|=1}^{2}\sum_{i\in\mathcal{C}_j}\sum_{\alpha\in\mathcal{I}_{q_1,q_2}}\frac{1}{\alpha!}(\dboijn)^{\alpha_1}\prod_{\ell=1}^{K-1}(\daijnl)^{\alpha_{2\ell}}(\dbijnl)^{\alpha_{3\ell}}\\
    &\times X^{q_1}\exp((\boj)^{\top}X)\frac{\partial^{|q_2|}f}{\partial h_1^{q_{21}}\ldots\partial h_{K-1}^{q_{2(K-1)}}}(Y=s|X;a^*_j,b^*_j)+R_1(X,Y).
\end{align*}
Similarly, by means of second order Taylor expansion, $A_{n,2}$ is expressed as follows:
\begin{align*}
    A_{n,2}&=\sum_{j:|\mathcal{C}_j|>1}\sum_{|q_1|+|q_2|=1}^{4}\sum_{i\in\mathcal{C}_j}\sum_{\alpha\in\mathcal{I}_{q_1,q_2}}\frac{1}{\alpha!}(\dboijn)^{\alpha_1}\prod_{\ell=1}^{K-1}(\daijnl)^{\alpha_{2\ell}}(\dbijnl)^{\alpha_{3\ell}}\\
    &\times X^{q_1}\exp((\boj)^{\top}X)\frac{\partial^{|q_2|}f}{\partial h_1^{q_{21}}\ldots\partial h_{K-1}^{q_{2(K-1)}}}(Y=s|X;a^*_j,b^*_j)+R_2(X,Y),
\end{align*}
where $R_2(X,Y)$ is also a Taylor remainder such that $R_2(X,Y)/\mathcal{D}_2(G_n,G_*)\to0$ as $n\to\infty$. By employing the same arguments for decomposing $A_n$ to $B_n$, we get
\begin{align*}
    B_{n}&=\sum_{j:|\mathcal{C}_j|=1}\sum_{|\gamma|=1}\sum_{i\in\mathcal{C}_j}\frac{\exp(\bzin)}{\gamma!}(\dboijn)^{\gamma}\times X^{\gamma}\exp((\boj)^{\top}X)g_{G_n}(Y=s|X)+R_{3}(X,Y)\\
    &+\sum_{j:|\mathcal{C}_j|>1}\sum_{|\gamma|=1}^{2}\sum_{i\in\mathcal{C}_j}\frac{\exp(\bzin)}{\gamma!}(\dboijn)^{\gamma}\times X^{\gamma}\exp((\boj)^{\top}X)g_{G_n}(Y=s|X)+R_{4}(X,Y).
\end{align*}
Putting the above results together, we obtain that
\begin{align*}
    T_{n}(s)&=\sum_{j=1}^{k_*}\sum_{|q_1|+|q_2|=0}^{2+2\cdot\mathbf{1}_{\{|\mathcal{C}_j|>1\}}}U^n_{q_1,q_2}(j)\times X^{q_1}\exp((\boj)^{\top}X)\frac{\partial^{|q_2|}f}{\partial h_1^{q_{21}}\ldots\partial h_{K-1}^{q_{2(K-1)}}}(Y=s|X;a^*_j,b^*_j)\\
    &+\sum_{j=1}^{k_*}\sum_{|\gamma|=0}^{1+\mathbf{1}_{\{|\mathcal{C}_j|>1\}}}W^n_{\gamma}(j)\times X^{\gamma}\exp((\boj)^{\top}X)g_{G_n}(Y=s|X) + R(X,Y),
\end{align*}
where $R(X,Y)$ is the sum of Taylor remainders such that $R(X,Y)/\mathcal{D}_2(X,Y)\to0$ as $n\to\infty$ and
\begin{align*}
    U^n_{q_1,q_2}(j)&=\begin{cases}
        \sum_{i\in\mathcal{C}_j}\sum_{\alpha\in\mathcal{I}_{q_1,q_2}}\frac{\exp(\bzin)}{\alpha!}(\dboijn)^{\alpha_1}\prod_{\ell=1}^{K-1}(\daijnl)^{\alpha_{2\ell}}(\dbijnl)^{\alpha_{3\ell}}, \quad (q_1,q_2)\neq(\zerod,\mathbf{0}_{K-1}),\\
        \sum_{i\in\mathcal{C}_j}\exp(\bzin)-\exp(\bzj), \hspace{5.7cm} (q_1,q_2)=(\zerod,\mathbf{0}_{K-1)},
    \end{cases}
\end{align*}
and 
\begin{align*}
    W^n_{\gamma}(j)&=\begin{cases}
        -\sum_{i\in\mathcal{C}_j}\frac{\exp(\bzin)}{\gamma!}(\dboijn)^{\gamma}, \hspace{1cm} |\gamma|\neq \zerod,\\
        -\sum_{i\in\mathcal{C}_j}\exp(\bzin)+\exp(\bzj), \hspace{0.5cm} |\gamma|=\zerod,
    \end{cases}
\end{align*}
for any $j\in[k_*]$.

\textbf{Step 2.} In this step, we will prove by contradiction that at least one among $U^n_{q_1,q_2}(j)/\mathcal{D}_2(G_n,G_*)$ and $W^n_{\gamma}(j)/\mathcal{D}_2(G_n,G_*)$ does not vanish as $n$ tends to infinity. Assume that all of them approach zero, then by taking the summation of $|U^n_{q_1,q_2}(j)|/\mathcal{D}_2(G_n,G_*)$ for $j\in[k_*]:|\mathcal{C}_j|=1$, $q_1\in\{e_1,e_2,\ldots,e_d\}$ and $q_2=\mathbf{0}_{K-1}$, where $e_i:=(0,\ldots,0,\underbrace{1}_{\textit{i-th}},0,\ldots,0)\in\mathbb{R}^d$, we achieve that
\begin{align}
    \label{eq:beta1_vanish_one}
    \frac{1}{\mathcal{D}_2(G_n,G_*)}\cdot\sum_{j:|\mathcal{C}_j|=1}\sum_{i\in\mathcal{C}_j}\exp(\bzin)\cdot\|\dboijn\|_1\to 0.
\end{align}
Similarly, for $q_1=\zerod$ and $q_{2}\in\{e'_1,e'_2,\ldots,e'_{K-1}\}$, where $e'_{\ell}:=(0,\ldots,0,\underbrace{1}_{\ell\textit{-th}},0,\ldots,0)\in\mathbb{R}^{K-1}$, we have
\begin{align}
    \label{eq:a_vanish_one}
    \frac{1}{\mathcal{D}_2(G_n,G_*)}\cdot\sum_{j:|\mathcal{C}_j|=1}\sum_{i\in\mathcal{C}_j}\sum_{\ell=1}^{K-1}\exp(\bzin)\cdot|\daijnl|\to 0.
\end{align}
On the other hand, for $q_1\in\{e_1,e_2,\ldots,e_d\}$ and $q_{2}\in\{e'_1,e'_2,\ldots,e'_{K-1}\}$, we obtain that
\begin{align}
    \label{eq:b_vanish_one}
    \frac{1}{\mathcal{D}_2(G_n,G_*)}\cdot\sum_{j:|\mathcal{C}_j|=1}\sum_{i\in\mathcal{C}_j}\sum_{\ell=1}^{K-1}\exp(\bzin)\cdot\|\dbijnl\|_1\to 0.
\end{align}
Combine the limits in equations~\eqref{eq:beta1_vanish_one}, \eqref{eq:a_vanish_one} and \eqref{eq:b_vanish_one}, we get
\begin{align*}
    \frac{1}{\mathcal{D}_2(G_n,G_*)}\cdot\sum_{j:|\mathcal{C}_j|=1}\sum_{i\in\mathcal{C}_j}\exp(\bzin)\cdot\Big[\|\dboijn\|_1+\sum_{\ell=1}^{K-1}(|\daijnl|+\|\dbijnl\|_1)\Big]\to0.
\end{align*}
Due to the topological equivalence between $1$-norm and $2$-norm, the above limit is equivalent to
\begin{align}
    \label{eq:vanish_one}
    \frac{1}{\mathcal{D}_2(G_n,G_*)}\cdot\sum_{j:|\mathcal{C}_j|=1}\sum_{i\in\mathcal{C}_j}\exp(\bzin)\cdot\Big[\|\dboijn\|+\sum_{\ell=1}^{K-1}(|\daijnl|+\|\dbijnl\|)\Big]\to0.
\end{align}
Next, we consider the summation of $|U^n_{q_1,q_2}(j)|/\mathcal{D}_2(G_n,G_*)$ for $j\in[k_*]:|\mathcal{C}_j|>1$, $q_1\in\{2e_1,2e_2,\ldots,2e_d\}$ and $q_2=\mathbf{0}_{K-1}$, which leads to 
\begin{align}
    \label{eq:beta1_vanish_two}
    \frac{1}{\mathcal{D}_2(G_n,G_*)}\cdot\sum_{j:|\mathcal{C}_j|>1}\sum_{i\in\mathcal{C}_j}\exp(\bzin)\cdot\|\dboijn\|^2\to 0.
\end{align}
For $q_1=\zerod$ and $q_{2}\in\{2e'_1,2e'_2,\ldots,2e'_{K-1}\}$, we have
\begin{align}
    \label{eq:a_vanish_two}
     \frac{1}{\mathcal{D}_2(G_n,G_*)}\cdot\sum_{j:|\mathcal{C}_j|>1}\sum_{i\in\mathcal{C}_j}\exp(\bzin)\cdot\sum_{\ell=1}^{K-1}|\daijnl|^2\to0.
\end{align}
Meanwhile, for $q_1\in\{2e_1,2e_2,\ldots,2e_d\}$ and $q_{2}\in\{2e'_1,2e'_2,\ldots,2e'_{K-1}\}$, we get 
\begin{align}
    \label{eq:b_vanish_two}
    \frac{1}{\mathcal{D}_2(G_n,G_*)}\cdot\sum_{j:|\mathcal{C}_j|>1}\sum_{i\in\mathcal{C}_j}\exp(\bzin)\cdot\sum_{\ell=1}^{K-1}\|\dbijnl\|^2\to0.
\end{align}
It follows from equations~\eqref{eq:beta1_vanish_two}, \eqref{eq:a_vanish_two} and \eqref{eq:b_vanish_two} that
\begin{align}
    \label{eq:vanish_two}
    \frac{1}{\mathcal{D}_2(G_n,G_*)}\cdot\sum_{j:|\mathcal{C}_j|>1}\sum_{i\in\mathcal{C}_j}\exp(\bzin)\cdot\Big[\|\dboijn\|^2+\sum_{\ell=1}^{K-1}(|\daijnl|^2+\|\dbijnl\|^2)\Big]\to0.
\end{align}
Note that
\begin{align}
    \label{eq:zero_vanish}
    \sum_{j=1}^{k_*}\frac{|U^n_{\zerod,\mathbf{0}_{K-1}}(j)|}{\mathcal{D}_2(G_n,G_*)}=\frac{1}{\mathcal{D}_2(G_n,G_*)}\cdot\sum_{j=1}^{k_*}\left|\sum_{i\in\mathcal{C}_j}\exp(\bzin)-\exp(\bzj)\right|\to0.
\end{align}
By taking the sum of limits in equations~\eqref{eq:vanish_one}, \eqref{eq:vanish_two} and \eqref{eq:zero_vanish}, we deduce that $1=\mathcal{D}_2(G_n,G_*)/\mathcal{D}_2(G_n,G_*)\to 0$ as $n\to\infty$, which is a contradiction. Thus, at least one among the limits of $U^n_{q_1,q_2}(j)/\mathcal{D}_2(G_n,G_*)$ and $W^n_{\gamma}(j)/\mathcal{D}_2(G_n,G_*)$ is non-zero.

\textbf{Step 3.} Finally, we will leverage Fatou's lemma to point out a contradiction to the result in Step 2. 

Let us denote by $m_n$ the maximum of the absolute values of $U^n_{q_1,q_2}(j)/\mathcal{D}_2(G_n,G_*)$ and $W^n_{\gamma}(j)/\mathcal{D}_2(G_n,G_*)$ for $j\in[k_*]$, $0\leq|q_1|+|q_2|\leq 2+2\cdot\mathbf{1}_{\{|\mathcal{C}_j|>1\}}$ and $0\leq|\gamma|\leq 1+\mathbf{1}_{\{|\mathcal{C}_j|>1\}}$. Then, it follows from the Fatou's lemma that
\begin{align*}
    0=\lim_{n\to\infty}\frac{\bbE_X[2V(g_{G_n}(\cdot|X),g_{G_*}(\cdot|X))]}{m_n\mathcal{D}_2(G_n,G_*)}\geq \int\sum_{s=1}^{K}\liminf_{n\to\infty}\frac{|g_{G_n}(Y=s|X)-g_{G_*}(Y=s|X)|}{m_n\mathcal{D}_2(G_n,G_*)}~\dint X\geq 0.
\end{align*}
As a result, we get that $|g_{G_n}(Y=s|X)-g_{G_*}(Y=s|X)|/[m_n\mathcal{D}_2(G_n,G_*)]$ converges to zero, which implies that $T_n(s)/[m_n\mathcal{D}_2(G_n,G_*)]\to0$ as $n\to\infty$ for any $s\in[K]$ and almost surely $X$. Let $U^n_{q_1,q_2}(j)/[m_n\mathcal{D}_2(G_n,G_*)]\to \tau_{q_1,q_2}(j)$ and $W^n_{\gamma}(j)\to\eta_{\gamma}(j)$ as $n$ approaches infinity, then the previous result indicates that
\begin{align}
    \label{eq:linear_independent_equation}
    \sum_{j=1}^{k_*}\sum_{|q_1|+|q_2|=0}^{2+2\cdot\mathbf{1}_{\{|\mathcal{C}_j|>1\}}}\tau_{q_1,q_2}(j)\times X^{q_1}\exp((\boj)^{\top}X)\frac{\partial^{|q_2|}f}{\partial h_1^{q_{21}}\ldots\partial h_{K-1}^{q_{2(K-1)}}}(Y=s|X;a^*_{j},b^*_{j})\nonumber\\
    +\sum_{j=1}^{k_*}\sum_{|\gamma|=0}^{1+\mathbf{1}_{\{|\mathcal{C}_j|>1\}}}\eta_{\gamma}(j)\times X^{\gamma}\exp((\boj)^{\top}X)g_{G_*}(Y=s|X)=0.
\end{align}
for any $s\in[K]$ and almost surely $X$. Here, at least one among $\tau_{q_1,q_2}(j)$ and $\eta_{\gamma}(j)$ is different from zero. Assume the set
\begin{align}
    \mathcal{F}:&=\Bigg\{X^{q_1}\exp((\boj)^{\top}X)\frac{\partial^{|q_2|}f}{\partial h_1^{q_{21}}\ldots\partial h_{K-1}^{q_{2(K-1)}}}(Y=s|X;a^*_j,b^*_j):j\in[k_*],0\leq |q_1|+|q_2|\leq 2+2\cdot\mathbf{1}_{\{|\mathcal{C}_j|>1\}}\Bigg\}\nonumber\\
    \label{eq:F_set}
    &\cup \Big\{X^{\gamma}\exp((\boj)^{\top}X)g_{G_*}(Y=s|X):j\in[k_*], \ 0\leq|\gamma|\leq 1+\mathbf{1}_{\{|\mathcal{C}_j|>1\}}\Big\}
\end{align}
is linearly independent, we deduce that $\tau_{q_1,q_2}(j)=\eta_{\gamma}(j)=0$ for any $j\in[k_*]$, $0\leq|q_1|+|q_2|\leq 2+2\cdot\mathbf{1}_{\{|\mathcal{C}_j|>1\}}$ and $0\leq|\gamma|\leq 1+\mathbf{1}_{\{|\mathcal{C}_j|>1\}}$, which is a contradiction.

Thus, it suffices to show that $\mathcal{F}$ is a linearly independent set to attain the local inequality in equation~\eqref{eq:local_structure_one}. In particular, assume that equation~\eqref{eq:linear_independent_equation} holds true for any $s\in[K]$ and almost surely $X$, we will show that $\tau_{q_1,q_2}(j)=\eta_{\gamma}(j)=0$ for any $j\in[k_*]$, $0\leq|q_1|+|q_2|\leq 2+2\cdot\mathbf{1}_{\{|\mathcal{C}_j|>1\}}$ and $0\leq|\gamma|\leq 1+\mathbf{1}_{\{|\mathcal{C}_j|>1\}}$. Firstly, we rewrite this equation as follows:
\begin{align*}
    \sum_{j=1}^{k_*}\sum_{|\omega|=0}^{1+\mathbf{1}_{\{|\mathcal{C}_j|>1\}}}\Bigg[\sum_{|q_2|=0}^{2+2\cdot\mathbf{1}_{\{|\mathcal{C}_j|>1\}}-|\omega|}\tau_{q_1,q_2}(j)\frac{\partial^{|q_2|}f}{\partial h_1^{q_{21}}\ldots\partial h_{K-1}^{q_{2(K-1)}}}(Y=s|X;a^*_j,b^*_j)&+\eta_{\omega}(j)g_{G_*}(Y=s|X)\Bigg]\\
    &\times X^{\omega}\exp((\boj)^{\top}X) = 0,
\end{align*}
for any $s\in[K]$ and almost surely $X$. Note that $\beta^*_{11},\beta^*_{12},\ldots,\beta^*_{1k_*}$ are $k_*$ different values, therefore, it can be seen that $\{X^{\omega}\exp((\boj)^{\top}X):j\in[k_*],\ 0\leq|\omega|\leq 1 +\mathbf{1}_{\{|\mathcal{C}_j|>1\}}\}$ is a linearly independent set. As a result, we obtain that
\begin{align*}
    \sum_{|q_2|=0}^{2+2\cdot\mathbf{1}_{\{|\mathcal{C}_j|>1\}}-|\omega|}\tau_{q_1,q_2}(j)\frac{\partial^{|q_2|}f}{\partial h_1^{q_{21}}\ldots\partial h_{K-1}^{q_{2(K-1)}}}(Y=s|X;a^*_j,b^*_j)+\eta_{\omega}(j)g_{G_*}(Y=s|X)=0,
\end{align*}
for any $j\in[k_*]$, $0\leq|\omega|\leq 1+\mathbf{1}_{\{|\mathcal{C}_j|>1\}}$, $s\in[K]$ and almost surely $X$. Similarly, the following set is linearly independent:
\begin{align*}
    \left\{\dfrac{\partial^{|q_2|}f}{\partial h_1^{q_{21}}\ldots\partial h_{K-1}^{q_{2(K-1)}}}(Y=s|X;a^*_j,b^*_j), \ g_{G_*}(Y=s|X):0\leq|q_2|\leq 2+2\cdot\mathbf{1}_{\{|\mathcal{C}_j|>1\}}-|\omega|\right\},
\end{align*}
we achieve that $\tau_{q_1,q_2}(j)=\eta_{\gamma}(j)=0$ for any $j\in[k_*]$, $0\leq|q_1|+|q_2|\leq 2+2\cdot\mathbf{1}_{\{|\mathcal{C}_j|>1\}}$ and $0\leq|\gamma|\leq 1+\mathbf{1}_{\{|\mathcal{C}_j|>1\}}$, which completes the proof of local structure.

\textbf{Global Structure}: As the local inequality in equation~\eqref{eq:local_structure_one} holds true, there exists a positive constant $\varepsilon'$ that satisfies
\begin{align*}
    \inf_{\substack{G\in\mathcal{O}_k(\Theta),\\ \mathcal{D}_2(G,G_*)\leq\varepsilon'}}\bbE_X[V(g_{G}(\cdot|X),g_{G_*}(\cdot|X))]/\mathcal{D}_{2}(G,G_*)>0.
\end{align*}
Therefore, it is sufficient to demonstrate that
\begin{align}
    \label{eq:global_structure_one}
    \inf_{\substack{G\in\mathcal{O}_k(\Theta),\\ \mathcal{D}_2(G,G_*)>\varepsilon'}}\bbE_X[V(g_{G}(\cdot|X),g_{G_*}(\cdot|X))]/\mathcal{D}_{2}(G,G_*)>0.
\end{align}
Assume by contrary that this inequality does not hold, then we can find a sequence $G'_n\in\mathcal{O}_k(\Theta)$ such that $\mathcal{D}_2(G'_n,G_*)>\varepsilon'$ and $\bbE_X[V(g_{G'_n}(\cdot|X),g_{G_*}(\cdot|X))]\to0$ as $n$ tends to infinity. It is worth noting that $\Theta$ is a compact set, therefore, we can replace $G'_n$ by its subsequence which converges to some mixing measure $G'\in\mathcal{O}_k(\Theta)$. Thus, we get that $\mathcal{D}_2(G',G_*)>\varepsilon'$. On the other hand, according to Fatou's lemma,
\begin{align*}
    0=\lim_{n\to\infty}\frac{\bbE_X[2V(g_{G'_n}(\cdot|X),g_{G_*}(\cdot|X))]}{\mathcal{D}_{2}(G'_n,G_*)}\geq \int \sum_{s=1}^{K}\liminf_{n\to\infty}|g_{G'_n}(Y=s|X)-g_{G_*}(Y=s|X)|~\dint X.
\end{align*}
Consequently, it follows that 
\begin{align*}
    \int \sum_{s=1}^{K}|g_{G'}(Y=s|X)-g_{G_*}(Y=s|X)|~\dint X=0,
\end{align*}
which means that $g_{G'}(Y=s|X)=g_{G_*}(Y=s|X)$ for any $s\in[K]$ for almost surely $X$. Recall that the softmax gating multinomial logistic mixture of experts is identifiable, we deduce that $G'\equiv G$, which contradicts the results that $\mathcal{D}_2(G',G_*)>\varepsilon'$. Hence, we obtain the global inequality in equation~\eqref{eq:global_structure_one} and complete the proof.

\subsection{Proof of Theorem~\ref{theorem:standard_two}}
\label{appendix:standard_two}
In Appendix~\ref{appendix:main_proof}, we present the proof of Theorem~\ref{theorem:standard_two} given the result of Proposition~\ref{prop:TV_prop}. Then, we provide the proof of Proposition~\ref{prop:TV_prop} in Appendix~\ref{appendix:TV_prop}.
\subsubsection{Main Proof}
\label{appendix:main_proof}
Given a fixed constant $M_1>0$ and a sufficiently small $\varepsilon>0$ that we will choose later, it follows from the result of Proposition~\ref{prop:TV_prop} that there exists a mixing measure $G'_*\in\mathcal{O}_k(\Theta)$ that satisfies $\mathcal{D}_r(G'_*,G_*)=2\varepsilon$ and $\bbE_X[V(g_{G'_*}(\cdot|X),g_{G_*}(\cdot|X))\leq M_1\varepsilon$. Note that for any sequence $\overline{G}_n\in\mathcal{O}_k(\Theta)$, we have
    \begin{align*}
        2\max_{G\in\{G'_*,G_*\}}\bbE_{g_{G}}[\mathcal{D}_r(\overline{G}_n,G)]\geq \bbE_{g_{G_*}}[\mathcal{D}_r(\overline{G}_n,G_*)]+\bbE_{g_{G'_*}}[\mathcal{D}_r(\overline{G}_n,G'_*)],
    \end{align*}
    where $\bbE_{g_{G}}$ denotes the expectation taken with respect to the product measure with density $g^n_{G}$. Moreover, since the loss $\mathcal{D}_r$ satisfies the weak triangle inequality, we can find a constant $M_2>0$ such that
    \begin{align*}
        \mathcal{D}_r(\overline{G}_n,G_*)+\mathcal{D}_r(\overline{G}_n,G'_*)\geq M_2\mathcal{D}_r(G_*,G'_*)=2M_2\varepsilon.
    \end{align*}
    As a result, we observe that
    \begin{align*}
        \max_{G\in\{G_*,G'_*\}}\bbE_{g_{G}}[\mathcal{D}_r(\overline{G}_n,G)]&\geq\frac{1}{2}\Big( \bbE_{g_{G_*}}[\mathcal{D}_r(\overline{G}_n,G_*)]+\bbE_{g_{G'_*}}[\mathcal{D}_r(\overline{G}_n,G'_*)]\Big)\\
        &\geq M_2\varepsilon\cdot\inf_{f_1,f_2}\left(\bbE_{g_{G_*}}[f_1]+\bbE_{g_{G'_*}}[f_2]\right).
    \end{align*}
    Here, $f_1$ and $f_2$ in the above infimum are measurable functions in terms of $X_1,X_2,\ldots,X_n$ that satisfy $f_1+f_2=1$. By the definition of Total Variation distance, the above infimum value is equal to $1-\bbE_X[V(g^n_{G_*}(\cdot|X),g^n_{G'_*}(\cdot|X))]$. Therefore, we obtain that
    \begin{align*}
        \max_{G\in\{G_*,G'_*\}}\bbE_{p_{G}}[\mathcal{D}_r(\overline{G}_n,G)]&\geq M_2\varepsilon\Big(1-\bbE_X[V(g^n_{G_*}(\cdot|X),g^n_{G'_*}(\cdot|X))]\Big)\\
        &\geq M_2\varepsilon\Big[1-\sqrt{1-(1-M_1^2\varepsilon^2)^n}\Big].
    \end{align*}
    By choosing $\varepsilon=n^{-1/2}/M_1$, we have $M_1^2\varepsilon^2=\frac{1}{n}$, which implies that
    \begin{align*}
         \sup_{G\in\mathcal{O}_k(\Theta)\setminus\mathcal{O}_{k_*-1}(\Theta)}\bbE_{g_{G}}[\mathcal{D}_r(\overline{G}_n,G)]\geq\max_{G\in\{G_*,G'_*\}}\bbE_{g_{G}}[\mathcal{D}_r(\overline{G}_n,G)]&\gtrsim n^{-1/2},
    \end{align*}
    for any mixing measure $\overline{G}_n\in\mathcal{O}_k(\Theta)$. Hence, we reach the conclusion of Theorem~\ref{theorem:standard_two}, which says that
    \begin{align*}
        \inf_{\overline{G}_n\in\mathcal{O}_k(\Theta)}\sup_{G\in\mathcal{O}_k(\Theta)\setminus\mathcal{O}_{k_*-1}(\Theta)}\bbE_{g_{G}}[\mathcal{D}_r(\overline{G}_n,G)]\gtrsim n^{-1/2},
    \end{align*}
for any $r\geq 1$.
\subsubsection{Proof of Proposition~\ref{prop:TV_prop}}
\label{appendix:TV_prop}
From the conditions of Regime 2, we may assume without loss of generality that $b^*_{1\ell}=\zerod$ for any $\ell\in[K-1]$. To reach the conclusion of Proposition~\ref{prop:TV_prop}, we need to find a sequence $G_n\in\mathcal{O}_k(\Theta)$ such that $\mathcal{D}_r(G_n,G_*)$ and $V(p_{G_n},p_{G_*})/\mathcal{D}_r(G_n,G_*)$ both tend to zero when $n$ approaches infinity. For that purpose, we choose $G_n=\sum_{i=1}^{k_*+1}\exp(\bzin)\delta_{(\boin,a^n_{11},\ldots,a^n_{1(K-1)},b^n_{11},\ldots,b^n_{1(K-1)}})$ where 
\begin{itemize}
    \item $\exp(\beta^n_{01})=\exp(\beta^n_{02})=\frac{1}{2}\exp(\beta^*_{01})-\frac{t_n}{2}$, $\exp(\bzin)=\exp(\beta^*_{0(i-1)})$ for any $3\leq i\leq k_*+1$;
    \item $\beta^n_{11}=\beta^n_{12}=\beta^*_{11}+{c_n}\mathbf{1}_d$, $\beta^n_{1i}=\beta^*_{1(i-1)}$ for any $3\leq i\leq k_*+1$;
    \item $a^n_{1\ell}=a^n_{2\ell}=a^*_{1\ell}+{c_n}$, $a^n_{i\ell}=a^*_{(i-1)\ell}$ for any $\ell\in[K-1]$ and $3\leq i\leq k_*+1$;
    \item $b^n_{1\ell}=b^n_{2\ell}=b^*_{1\ell}$, $b^n_{i\ell}=b^*_{(i-1)\ell}$ for any $\ell\in[K-1]$ and $3\leq i\leq k_*+1$,
\end{itemize}
where $t_n,c_n>0$ will be chosen later such that ${t_n}\to0$ and $c_n=\mathcal{O}(t_n)$ as $n\to\infty$. Then, it can be verified that
\begin{align}
    \label{eq:D_r_formulation}
    \mathcal{D}_r(G_n,G_*)=(K-1)\Big[\exp(\beta^*_{01})-{t_n}\Big]c^{r}_n(1+d^{r/2})+{t_n}.
\end{align}
Now, we will show that $\bbE_X[V(g_{G_n}(\cdot|X),g_{G_*}(\cdot|X))]/\mathcal{D}_r(G_n,G_*)$ vanishes as $n\to\infty$. Let us reconsider the quantity $T_n(s)=\Big[\sum_{j=1}^{k_*}\exp((\boj)^{\top}X+\bzj)\Big]\cdot\Big[g_{G_n}(Y=s|X)-g_{G_*}(Y=s|X)\Big]$ in equation~\eqref{eq:Tn_decomposition} under the above setting of $G_n$ as follows:
\begin{align*}
    T_n(s)&=\sum_{i=1}^{2}\exp(\bzin)[u(Y=s|X;\beta^n_{1i},a^n_i,b^n_i)-u(Y=s|X;\beta^*_{11},a^*_1,b^*_1)]\\
    &-\sum_{i=1}^{2}\exp(\bzin)[v(Y=s|X;\boin)-v(Y=s|X;\beta^*_{11})]\\
    &+\left(\sum_{i=1}^{2}\exp(\bzin)-\exp(\beta^*_{01})\right)[u(Y=s|X;\beta^*_{11},a^*_1,b^*_1)-v(Y=s|X;\beta^*_{11})]\\
    &:=A_n-B_n+E_n,
\end{align*}
where $u(Y=s|X;\beta_1,a,b):=\exp(\beta_1^{\top}X)f(Y=s|X;a,b)$ and $v(Y=s|X;\beta_1)=\exp(\beta_1^{\top}X)g_{G_n}(Y=s|X)$ for any $s
\in[K]$. Recall that we denote $h_{\ell}(X,a_{\ell},b_{\ell}):=a_{\ell}+b_{\ell}^{\top}X$ for any $\ell\in[K]$ and 
\begin{align*}
    f(Y=s|X;a_{i},b_{i})=\dfrac{\exp(a_{is}+b_{is}^{\top}X)}{\sum_{\ell=1}^K\exp(a_{i\ell}+b_{i\ell}^{\top}X)}=\dfrac{\exp(h_s(X,a_{is},b_{is}))}{\sum_{\ell=1}^{K}\exp(h_{\ell}(X,a_{i\ell},b_{i\ell}))}.
\end{align*}
Then, by means of Taylor expansion up to order $r$, we have
\begin{align*}
    A_n&=\sum_{i=1}^{2}\exp(\bzin)\sum_{|\alpha|=1}^{r}\frac{1}{\alpha!}(\beta^n_{1i}-\beta^*_{11})^{\alpha_1}\prod_{\ell=1}^{K-1}(a^n_{i\ell}-a^*_{1\ell})^{\alpha_{2\ell}}(b^n_{i\ell}-b^*_{1\ell})^{\alpha_{3\ell}}\cdot X^{\alpha_1+\sum_{\ell=1}^{K-1}\alpha_{3\ell}}\exp((\beta^*_{11})^{\top}X)\\
    &\times\frac{\partial^{\sum_{\ell=1}^{K-1}(\alpha_{2\ell}+|\alpha_{3\ell}|)} f}{\partial h_1^{\alpha_{21}+|\alpha_{31}|}\ldots\partial h_{K-1}^{\alpha_{2(K-1)}+|\alpha_{3(K-1)}|}}(Y=s|h^*_{11},\ldots,h^*_{1K}) + R_5(X,Y),
\end{align*}
Here, $R_5(X,Y)$ is a Taylor remainder that satisfies $R_5(X,Y)/\mathcal{D}_r(G_n,G_*)\to0$ as $n\to\infty$. Since $b^*_{1\ell}=\zerod$ for any $\ell\in[K-1]$, the derivatives of $f$ with respect to $h_1,\ldots,h_{K-1}$ in the above representation of $A_n$ are constants depending on $(\alpha_{2\ell})_{\ell=1}^{K-1}$ and $(\alpha_{3\ell})_{\ell=1}^{K-1}$. Therefore, we can denote them as $C_{\alpha_2,\alpha_3}$. 

Additionally, since $\beta^n_{1i}-\beta^*_{11}=\zerod$ and $b^n_{i\ell}-b^*_{1\ell}=\zerod$ for any $\ell\in[K-1]$ and $i\in\{1,2\}$, we can let $\alpha_1=\zerod$, $\alpha_{3\ell}=\zerod$, and rewrite $A_n$ as follows:
\begin{align*}
    A_n&=\sum_{|\alpha_1|+|\alpha_2|=1}^{r}\sum_{i=1}^{2}\frac{\exp(\bzin)}{\alpha_1!\alpha_{2}!}(\beta^n_{1i}-\beta^*_{11})^{\alpha_1}\prod_{\ell=1}^{K-1}(a^n_{i\ell}-a^*_{1\ell})^{\alpha_{2\ell}}X^{\alpha_1}\exp((\beta^*_{11})^{\top}X)C_{\alpha_2,\mathbf{0}} + R_5(X,Y)\\
    &=\sum_{|\alpha_1|+|\alpha_2|=1}^{r}\frac{\exp(\beta^*_{01})-{t_n}}{\alpha_1!\alpha_{2}!}\cdot{c^{|\alpha_1|+|\alpha_2|}_n}\cdot X^{\alpha_1}\exp((\beta^*_{11})^{\top}X)C_{\alpha_2,\mathbf{0}} + R_5(X,Y).
\end{align*}
Similarly, by applying the Taylor expansion up to order $r$ to $B_n$, we have
\begin{align*}
    B_n&=\sum_{|\gamma|=1}^{r}\sum_{i=1}^{2}\frac{1}{\gamma!}\exp(\bzin)(\beta^n_{1i}-\beta^*_{11})^{\gamma}\cdot X^{\gamma}\exp((\beta^*_{11})^{\top}X)g_{G_n}(Y=s|X) + R_6(X,Y)\\
    &=\sum_{|\gamma|=1}^{r}\frac{\exp(\beta^*_{01})-{t_n}}{\gamma!}\cdot{c^{|\gamma|}_n}\cdot X^{\gamma}\exp((\beta^*_{11})^{\top}X)g_{G_n}(Y=s|X) + R_6(X,Y),
\end{align*}
where $R_6(X,Y)$ is a Taylor remainder such that $R_6(X,Y)/\mathcal{D}_r(G_n,G_*)\to0$ as $n\to\infty$. 

Now, we aim to demonstrate that $(A_n+E_{n,1})/\mathcal{D}_r(G_n,G_*)\to0$ and $(B_n+E_{n,2})/\mathcal{D}_r(G_n,G_*)\to0$ as $n\to\infty$, where we define
\begin{align*}
    E_{n,1}&:=\left(\sum_{i=1}^{2}\exp(\bzin)-\exp(\beta^*_{01})\right)u(Y=s|X;\beta^*_{11},a^*_1,b^*_1),\\
    E_{n,2}&:=\left(\sum_{i=1}^{2}\exp(\bzin)-\exp(\beta^*_{01})\right)v(Y=s|X;\beta^*_{11}),
\end{align*}
in which $u(Y=s|X;\beta^*_{11},a^*_1,b^*_1)=\exp((\beta^*_{11})^{\top}X)f(Y=s|X;a^*_1,b^*_1)=\exp((\beta^*_{11})^{\top}X)C_{\mathbf{0},\mathbf{0}}$, and $0<C_{\mathbf{0},\mathbf{0}}<1$.

\textbf{Part 1.} Prove that $(A_n+E_{n,1})/\mathcal{D}_r(G_n,G_*)\to0$ as $n\to\infty$.

Since $\mathcal{X}$ is a bounded set, we assume that $\|X\|\leq B$ for some positive constant $M$. Then, we can verify that
\begin{align*}
    0\leq A_n+E_{n,1}-R_5(X,Y)\leq \sum_{|\alpha_1|+|\alpha_2|=0}^{r}L^n_{\alpha_1,\alpha_2}\exp((\beta^*_{11})^{\top}X),
\end{align*}
where we denote
\begin{align*}
    L^n_{\alpha_1,\alpha_2}:=\begin{cases}
        \dfrac{\exp(\beta^*_{01})-{t_n}}{\alpha_1!\alpha_{2}!}\cdot{c^{|\alpha_1|+|\alpha_2|}_n}\cdot B^{|\alpha_1|}C_{\alpha_2,\mathbf{0}}, \hspace{1cm}|\alpha_1|+|\alpha_2|>0,\\
        -t_nC_{\mathbf{0},\mathbf{0}}, \hspace{5.41cm} |\alpha_1|+|\alpha_2|=0.
    \end{cases}
\end{align*}
Note that for $\alpha_1,\alpha_2$ such that $2\leq|\alpha_1|+|\alpha_2|\leq r$, as $c_n=\mathcal{O}(t_n)$, we have $L^n_{\alpha_1,\alpha_2}/\mathcal{D}_r(G_n,G_*)\to0$ as $n\to\infty$. Now, we show that $\sum_{|\alpha_1|+|\alpha_2|=0}^{1}L^n_{\alpha_1,\alpha_2}=0$. Indeed, we have
\begin{align*}
    \sum_{|\alpha_1|+|\alpha_2|=0}^{1}L^n_{\alpha_1,\alpha_2}=[\exp(\beta^*_{01})-t_n]\cdot(C_{\mathbf{0},\mathbf{0}}B+ N)c_n-C_{\mathbf{0},\mathbf{0}}t_n,
\end{align*}
where $N:=\sum_{|\alpha_2|=1}C_{\alpha_2,\mathbf{0}}$. By setting $t_n=\frac{B}{nN}$ and $c_n=\frac{1}{nN\exp(\beta^*_{01})-B}$, the above sum reduces to zero. Thus, we get that
\begin{align*}
    {[A_n+E_{n,1}-R_5(X,Y)]}/{\mathcal{D}_r(G_n,G_*)}\to0,
\end{align*}
as $n\to\infty$. Recall that $R_5(X,Y)/\mathcal{D}_r(G_n,G_*)\to0$, we deduce that $(A_n+E_{n,1})/\mathcal{D}_r(G_n,G_*)\to0$ as $n\to\infty$.

\textbf{Part 2.} Prove that $(B_n+E_{n,2})/\mathcal{D}_r(G_n,G_*)\to0$ as $n\to\infty$.

It is worth noting that
\begin{align*}
    0\leq B_n+E_{n,2}-R_6(X,Y)\leq \sum_{|\gamma|=0}^{r}J^n_{\gamma}\exp((\beta^*_{11})^{\top}X)g_{G_n}(Y=s|X),
\end{align*}
where we denote 
\begin{align*}
    J^n_{\gamma}:=\begin{cases}
        \dfrac{\exp(\beta^*_{01})-t_n}{\gamma!}\cdot c_n^{|\gamma|}B^{|\gamma|}, \hspace{1cm} |\gamma|>0,\\
        -t_n, \hspace{4.05cm} |\gamma|=0.
    \end{cases}
\end{align*}
For $2\leq|\gamma|\leq r$, we have that $J^n_{\gamma}/\mathcal{D}_r(G_n,G_*)\to0$ as $n\to\infty$. Additionally, we have
\begin{align*}
    \sum_{|\gamma|=0}^{1}J^n_{\gamma}=[\exp(\beta^*_{01})-t_n]\cdot c_n\cdot B - t_n=0.
\end{align*}
As a result, we get that
\begin{align*}
    {[B_n+E_{n,2}-R_6(X,Y)]}/{\mathcal{D}_r(G_n,G_*)}\to0,
\end{align*}
as $n\to\infty$. Since $R_6(X,Y)/\mathcal{D}_r(G_n,G_*)\to0$, we obtain that $(B_n+E_{n,2})/\mathcal{D}_r(G_n,G_*)\to0$ as $n\to\infty$.

Putting the results of Part 1 and Part 2 together, we obtain that 
\begin{align*}
    \frac{T_n(s)}{\mathcal{D}_r(G_n,G_*)}=\frac{A_n-B_n+E_n}{\mathcal{D}_r(G_n,G_*)}\to0,
\end{align*}
which indicates that $\bbE_X[V(g_{G_n}(\cdot|X),g_{G_*}(\cdot|X
))]/\mathcal{D}_r(G_n,G_*)\to0$ as $n\to\infty$. Furthermore, it follows from equation~\eqref{eq:D_r_formulation} that $\mathcal{D}_r(G_n,G_*)\to0$ as $n\to\infty$. Hence, we reach the conclusion of Proposition~\ref{prop:TV_prop}.

\subsection{Proof of Theorem~\ref{theorem:modified_softmax}}
\label{appendix:modified_softmax}
To reach the desired conclusion in Theorem~\ref{theorem:modified_softmax}, we need to show the following key inequality:
\begin{align}
    \label{eq:original_inequality_modified}
    \inf_{G\in\mathcal{O}_{k}(\Theta)}\bbE_X[V(\gmod_{G}(\cdot|X),\gmod_{G_*}(\cdot|X))]/\mathcal{D}_{2}(G,G_*)>0,
\end{align}
which is then divided into two parts named local structure and global structure. Since the global structure can be argued similarly to that in Appendix~\ref{appendix:standard_one} with a note that the modified softmax gating multinomial logistic MoE model is identifiable (see Proposition~\ref{prop:modified_identifiability}), the proof for it is omitted in this appendix.

\textbf{Local Structure:} For this part, we use the proof by contradiction method to show that
\begin{align}
    \label{eq:local_structure_modified}
    \lim_{\varepsilon\to 0}\inf_{\substack{G\in\mathcal{O}_k(\Theta),\\ \mathcal{D}_2(G,G_*)\leq\varepsilon}}\bbE_X[V(\gmod_{G}(\cdot|X),\gmod_{G_*}(\cdot|X))]/\mathcal{D}_{2}(G,G_*)>0.
\end{align}
Assume that this local inequality does not hold, then by utilizing some derivations in Appendix~\ref{appendix:standard_one}, we proceed the three-step framework as follows:

\textbf{Step 1.} First of all, by deriving in the same fashion as in equation~\eqref{eq:Tn_decomposition}, we get the following decomposition of $\tmod_{n}(s):=\Big[\sum_{j=1}^{k_*}\exp((\boj)^{\top}X+\bzj)\Big]\cdot\Big[\gmod_{G_n}(Y=s|X)-\gmod_{G_*}(Y=s|X)\Big]$:
\begin{align*}
    \tmod_n(s)&=\sum_{j=1}^{k_*}\sum_{i\in\mathcal{C}_j}\exp(\bzin)\Big[\umod(Y=s|X;\boin,\ain,\bin)-\umod(Y=s|X;\boj,\aj,\bj)\Big]\nonumber\\
    &-\sum_{j=1}^{k_*}\sum_{i\in\mathcal{C}_j}\exp(\bzin)\Big[\vmod(Y=s|X;\boin)-\vmod(Y=s|X;\boj)\Big]\nonumber\\
    &+\sum_{j=1}^{k_*}\Big(\sum_{i\in\mathcal{C}_j}\exp(\bzin)-\exp(\bzj)\Big)\Big[\umod(Y=s|X;\boj,\aj,\bj)-\vmod(Y=s|X;\boj)\Big]\nonumber\\
    &:=A_n-B_n+E_n,
\end{align*}
where we define 
\begin{align*}
    \umod(Y=s|X;\beta_{1i},a_i,b_i)&:=\exp(\beta_{1i}^{\top}M(X))\cdot f(Y=s|X;a_{i},b_{i}),\\
    \vmod(Y=s|X;\beta_{1i})&:=\exp(\beta_{1i}^{\top}M(X))\cdot g_{G_n}(Y=s|X).
\end{align*}
for any $s\in[K]$. Next, we will apply first order and second order Taylor expansions to two terms in the following sum:
\begin{align*}
    A_{n}&=\sum_{j:|\mathcal{C}_j|=1}\sum_{i\in\mathcal{C}_j}\exp(\bzin)\Big[\umod(Y=s|X;\boin,\ain,\bin)-\umod(Y=s|X;\boj,\aj,\bj)\Big]\\
    &+\sum_{j:|\mathcal{C}_j|>1}\sum_{i\in\mathcal{C}_j}\exp(\bzin)\Big[\umod(Y=s|X;\boin,\ain,\bin)-\umod(Y=s|X;\boj,\aj,\bj)\Big]\\
    &:=A_{n,1} + A_{n,2}.
\end{align*}
For the first term, we get that
\begin{align*}
    A_{n,1}&=\sum_{j:|\mathcal{C}_j|>1}\sum_{i\in\mathcal{C}_j}\exp(\bzin)\sum_{|\alpha|=1}\frac{1}{\alpha!}(\dboijn)^{\alpha_1}\prod_{\ell=1}^{K-1}(\daijnl)^{\alpha_{2\ell}}(\dbijnl)^{\alpha_{3\ell}}\cdot[M(X)]^{\alpha_1} X^{\sum_{\ell=1}^{K-1}\alpha_{3\ell}}\\
    &\times \exp((\boj)^{\top}M(X))\cdot\frac{\partial^{\sum_{\ell=1}^{K-1}(\alpha_{2\ell}+|\alpha_{3\ell}|)} f}{\partial h_1^{\alpha_{21}+|\alpha_{31}|}\ldots\partial h_{K-1}^{\alpha_{2(K-1)}+|\alpha_{3(K-1)}|}}(Y=s|X;a^*_j,b^*_j) + \rmod_1(X,Y),
\end{align*}
where $\rmod_1(X,Y)$ is a Taylor remainder such that $\rmod_1(X,Y)/\mathcal{D}_2(G_n,G_*)\to0$ as $n\to\infty$. Let $q_2=(\alpha_{2\ell}+|\alpha_{3\ell}|)_{\ell=1}^{K-1}\in\mathbb{N}^{K-1}$, $q_3=\sum_{\ell=1}^{K-1}\alpha_{3\ell}\in\mathbb{N}^d$ and $q_4=\alpha_1\in\mathbb{N}^d$, we rewrite $A_{n,1}$ as
\begin{align*}
    A_{n,1}&=\sum_{j:|\mathcal{C}_j|=1}\sum_{|q_4|+|q_2|=1}\sum_{|q_3|=0}^{|q_2|}\sum_{i\in\mathcal{C}_j}\sum_{\alpha\in\mathcal{I}_{q_2,q_3,q_4}}\frac{\exp(\bzin)}{\alpha!}(\dboijn)^{\alpha_1}\prod_{\ell=1}^{K-1}(\daijnl)^{\alpha_{2\ell}}(\dbijnl)^{\alpha_{3\ell}}\\
    &\times[M(X)]^{q_4}X^{q_3}\exp((\boj)^{\top}M(X))\cdot\frac{\partial^{|q_2|} f}{\partial h_1^{q_{21}}\ldots\partial h_{K-1}^{q_{2(K-1)}}}(Y=s|X;a^*_j,b^*_j) + \rmod_1(X,Y),
\end{align*}
where we define
\begin{align*}
    \mathcal{I}_{q_2,q_3,q_4}:=\Big\{\alpha=(\alpha_1,\alpha_{21},\ldots,\alpha_{2(K-1)},\alpha_{31},\ldots,\alpha_{3(K-1)}):\alpha_1=q_4,\sum_{\ell=1}^{K-1}\alpha_{3\ell}=q_3,(\alpha_{2\ell}+|\alpha_{3\ell}|)_{\ell=1}^{K-1}=q_2\Big\}.
\end{align*}
For the second term $A_{n,2}$, we have
\begin{align*}
    A_{n,2}&=\sum_{j:|\mathcal{C}_j|>1}\sum_{|q_4|+|q_2|=1}^{2}\sum_{|q_3|=0}^{|q_2|}\sum_{i\in\mathcal{C}_j}\sum_{\alpha\in\mathcal{I}_{q_2,q_3,q_4}}\frac{\exp(\bzin)}{\alpha!}(\dboijn)^{\alpha_1}\prod_{\ell=1}^{K-1}(\daijnl)^{\alpha_{2\ell}}(\dbijnl)^{\alpha_{3\ell}}\\
    &\times[M(X)]^{q_4}X^{q_3}\exp((\boj)^{\top}M(X))\cdot\frac{\partial^{|q_2|} f}{\partial h_1^{q_{21}}\ldots\partial h_{K-1}^{q_{2(K-1)}}}(Y=s|X;a^*_j,b^*_j) + \rmod_2(X,Y),
\end{align*}
where $\rmod_2(X,Y)$ is a Taylor remainder such that $\rmod_2(X,Y)/\mathcal{D}_2(G_n,G_*)\to0$ when $n\to\infty$. 

Meanwhile, by arguing similarly, we can decompose $B_n$ as
\begin{align*}
    B_n&=\sum_{j:|\mathcal{C}_j|=1}\sum_{|\gamma|=1}\sum_{i\in\mathcal{C}_j}\frac{\exp(\bzin)}{\gamma!}(\dboijn)^{\gamma}\times [M(X)]^{\gamma}\exp((\boj)^{\top}M(X))g_{G_n}(Y=s|X)+\rmod_3(X,Y)\\
    &+\sum_{j:|\mathcal{C}_j|>1}\sum_{|\gamma|=1}^{2}\sum_{i\in\mathcal{C}_j}\frac{\exp(\bzin)}{\gamma!}(\dboijn)^{\gamma}\times [M(X)]^{\gamma}\exp((\boj)^{\top}M(X))g_{G_n}(Y=s|X)+\rmod_4(X,Y),
\end{align*}
where $\rmod_3(X,Y)$ and $\rmod_4(X,Y)$ are Taylor remainders such that their ratios to $\mathcal{D}_2(G_n,G_*)$ vanish as $n$ approaches infinity. 

Combine the above results, we deduce that $\tmod_n(s)$ can be represented as follows:
\begin{align}
    \tmod_n(s)&=\sum_{j=1}^{k_*}\sum_{|q_4|+|q_2|=0}^{1+\mathbf{1}_{\{|\mathcal{C}_j|>1\}}}\sum_{|q_3|=0}^{|q_2|}Z^n_{q_2,q_3,q_4}(j)\times [M(X)]^{q_4}X^{q_3}\exp((\boj)^{\top}M(X))\frac{\partial^{|q_2|} f}{\partial h_1^{q_{21}}\ldots\partial h_{K-1}^{q_{2(K-1)}}}(Y=s|X;a^*_j,b^*_j)\nonumber\\
    \label{eq:Tn_coefficient}
    &\hspace{2cm}+\sum_{j=1}^{k_*}\sum_{|\gamma|=0}^{1+\mathbf{1}_{\{|\mathcal{C}_j|>1\}}}W^n_{\gamma}(j)\times [M(X)]^{\gamma}\exp((\boj)^{\top}M(X))g_{G_n}(Y=s|X) + \rmod(X,Y),
\end{align}
where $\rmod(X,Y)$ is the sum of Taylor remainders such that $\rmod(X,Y)/\mathcal{D}_2(G_n,G_*)\to0$ as $n\to\infty$ and
\begin{align*}
    Z^n_{q_2,q_3,q_4}(j)&=\begin{cases}
        \sum_{i\in\mathcal{C}_j}\sum_{\alpha\in\mathcal{I}_{q_2,q_3,q_4}}\frac{\exp(\bzin)}{\alpha!}(\dboijn)^{\alpha_1}\prod_{\ell=1}^{K-1}(\daijnl)^{\alpha_{2\ell}}(\dbijnl)^{\alpha_{3\ell}}, \ (q_2,q_3,q_4)\neq(\mathbf{0}_{K-1},\zerod,\zerod),\\
        \sum_{i\in\mathcal{C}_j}\exp(\bzin)-\exp(\bzj), \hspace{5.8cm} (q_2,q_3,q_4)=(\mathbf{0}_{K-1},\zerod,\zerod),
    \end{cases}
\end{align*}
and 
\begin{align*}
    W^n_{\gamma}(j)&=\begin{cases}
        -\sum_{i\in\mathcal{C}_j}\frac{\exp(\bzin)}{\gamma!}(\dboijn)^{\gamma}, \hspace{1cm} |\gamma|\neq \zerod,\\
        -\sum_{i\in\mathcal{C}_j}\exp(\bzin)+\exp(\bzj), \hspace{0.5cm} |\gamma|=\zerod,
    \end{cases}
\end{align*}
for any $j\in[k_*]$. 

\textbf{Step 2.} Subsequently, we will show that at least one among $Z^n_{q_2,q_3,q_4}(j)/\mathcal{D}_2(G_n,G_*)$ does not approach zero as $n$ tends to infinity. Assume by contrary that all of them vanish as $n\to\infty$, then we consider some typical tuples $(q_2,q_3,q_4)$. Firstly, by taking the sum of $Z^n_{q_2,q_3,q_4}(j)$ for $j\in[k_*]:|\mathcal{C}_j|>1$ (resp. $j\in[k_*]:|\mathcal{C}_j|=1$) and  $(q_2,q_3,q_4)\in\{(\mathbf{0}_{K-1},\zerod,2e_i):i\in[d]\}$ (resp. $(q_2,q_3,q_4)\in\{(\mathbf{0}_{K-1},\zerod,e_i):i\in[d]\}$) where $e_i:=(0,\ldots,0,\underbrace{1}_{\textit{i-th}},0,\ldots,0)\in\mathbb{R}^d$, we get
\begin{align}
    \label{eq:modified_vanish_beta}
    \frac{1}{\mathcal{D}_2(G_n,G_*)}\cdot\sum_{j:|\mathcal{C}_j|>1}\sum_{i\in\mathcal{C}_j}\exp(\bzin)\|\dboijn\|^2&\to0,\nonumber\\
    \frac{1}{\mathcal{D}_2(G_n,G_*)}\cdot\sum_{j:|\mathcal{C}_j|=1}\sum_{i\in\mathcal{C}_j}\exp(\bzin)\|\dboijn\|&\to0.
\end{align}
For $(q_2,q_3,q_4)\in\{(\mathbf{0}_{K-1},2e_i,\zerod):i\in[d]\}$ (resp. $(q_2,q_3,q_4)\in\{(\mathbf{0}_{K-1},e_i,\zerod):i\in[d]\}$) where , we obtain that
\begin{align}
    \label{eq:modified_vanish_b}
     \frac{1}{\mathcal{D}_2(G_n,G_*)}\cdot\sum_{j:|\mathcal{C}_j|>1}\sum_{i\in\mathcal{C}_j}\exp(\bzin)\sum_{\ell=1}^{K-1}\|\dbijnl\|^2&\to0,\nonumber\\
    \frac{1}{\mathcal{D}_2(G_n,G_*)}\cdot\sum_{j:|\mathcal{C}_j|=1}\sum_{i\in\mathcal{C}_j}\exp(\bzin)\sum_{\ell=1}^{K-1}\|\dbijnl\|&\to0.
\end{align}
On the other hand, for $(q_2,q_3,q_4)\in\{(e'_{\ell},\zerod,\zerod):\ell\in[K-1]\}$ (resp. $(q_2,q_3,q_4)\in\{(e'_{\ell},\zerod,\zerod):\ell\in[K-1]\}$) where $e'_{\ell}:=(0,\ldots,0,\underbrace{1}_{\ell\textit{-th}},0,\ldots,0)\in\mathbb{R}^{K-1}$, we have
\begin{align}
    \label{eq:modified_vanish_a}
    \frac{1}{\mathcal{D}_2(G_n,G_*)}\cdot\sum_{j:|\mathcal{C}_j|>1}\sum_{i\in\mathcal{C}_j}\exp(\bzin)\sum_{\ell=1}^{K-1}\|\daijnl\|^2&\to0,\\
    \frac{1}{\mathcal{D}_2(G_n,G_*)}\cdot\sum_{j:|\mathcal{C}_j|=1}\sum_{i\in\mathcal{C}_j}\exp(\bzin)\sum_{\ell=1}^{K-1}\|\daijnl\|&\to0.
\end{align}
Additionally, when $(q_2,q_3,q_4)=(\mathbf{0}_{K-1},\zerod,\zerod)$, it follows that
\begin{align}
    \label{eq:modified_vanish_zero}
    \sum_{j=1}^{k_*}\frac{|Z^n_{\mathbf{0}_{K-1},\zerod,\zerod}(j)|}{\mathcal{D}_2(G_n,G_*)}=\frac{1}{\mathcal{D}_2(G_n,G_*)}\cdot\sum_{j=1}^{k_*}\left|\sum_{i\in\mathcal{C}_j}\exp(\bzin)-\exp(\bzj)\right|\to0.
\end{align}
It is induced from the limits in equations~\eqref{eq:modified_vanish_beta}, \eqref{eq:modified_vanish_b}, \eqref{eq:modified_vanish_a} and \eqref{eq:modified_vanish_zero} that $1=\mathcal{D}_2(G_n,G_*)/\mathcal{D}_2(G_n,G_*)\to 0$ when $n\to\infty$, which is a contradiction. Thus, at least one among $Z^n_{q_2,q_3,q_4}(j)/\mathcal{D}_2(G_n,G_*)$ does not go to zero as $n\to\infty$.

\textbf{Step 3.} Now, we denote $\widetilde{m}_n$ as the maximum of the absolute values of $Z^n_{q_2,q_3,q_4}(j)/\mathcal{D}_2(G_n,G_*)$ and $W^n_{\gamma}(j)/\mathcal{D}_2(G_n,G_*)$ for any $j\in[k_*]$, $0\leq|q_2|+|q_4|\leq 1+\mathbf{1}_{\{|\mathcal{C}_j|>1\}}$, $0\leq|q_3|\leq|q_2|$ and $0\leq|\gamma|\leq 1+\mathbf{1}_{\{|\mathcal{C}_j|>1\}}$. Since at least one among $Z^n_{q_2,q_3,q_4}(j)/\mathcal{D}_2(G_n,G_*)$ does not go to zero as $n\to\infty$, we deduce that $\widetilde{m}_n\not\to0$, and therefore, $1/\widetilde{m}_n\not\to\infty$. Then, we denote
\begin{align*}
    Z^n_{q_2,q_3,q_4}(j)/[m_n\mathcal{D}_2(G_n,G_*)]&\to\widetilde{\tau}_{q_2,q_3,q_4}(j)\\
    W^n_{\gamma}(j)/\mathcal{D}_2(G_n,G_*)&\to\widetilde{\eta}_{\gamma}(j)
\end{align*}
as $n\to\infty$. Here, at least one among $\widetilde{\tau}_{q_2,q_3,q_4}(j)$ for $j\in[k_*]$, $0\leq|q_2|+|q_4|\leq 1+\mathbf{1}_{\{|\mathcal{C}_j|>1\}}$ and $0\leq|q_3|\leq|q_2|$ is non-zero. By invoking the Fatou's lemma, we have that
\begin{align*}
    0=\lim_{n\to\infty}\frac{\bbE_X[2V(\widetilde{g}_{G_n}(\cdot|X),\widetilde{g}_{G_*}(\cdot|X))]}{m_n\mathcal{D}_2(G_n,G_*)}\geq \int\sum_{s=1}^{K}\liminf_{n\to\infty}\frac{|\widetilde{g}_{G_n}(Y=s|X)-\widetilde{g}_{G_*}(Y=s|X)|}{m_n\mathcal{D}_2(G_n,G_*)}~\dint X\geq 0,
\end{align*}
which indicates that $[\widetilde{g}_{G_n}(Y=s|X)-\widetilde{g}_{G_*}(Y=s|X)]/[\widetilde{m}_n\mathcal{D}_2(G_n,G_*)]$ tends to zero as $n$ goes to infinity for any $s\in[K]$ and almost surely $X$. This result is equivalent to
\begin{align}
    \label{eq:zero_limit}
    \widetilde{T}_n(s)/[\widetilde{m}_n\mathcal{D}_2(G_n,G_*)]\to0,
\end{align}
as $n\to\infty$, for any $s\in[K]$. Putting the results in equations~\eqref{eq:Tn_coefficient} and \eqref{eq:zero_limit} together, we have
\begin{align}
    &\sum_{j=1}^{k_*}\sum_{|q_4|+|q_2|=0}^{1+\mathbf{1}_{\{|\mathcal{C}_j|>1\}}}\sum_{|q_3|=0}^{|q_2|}\widetilde{\tau}_{q_2,q_3,q_4}(j)\times [M(X)]^{q_4}X^{q_3}\exp((\boj)^{\top}M(X))\frac{\partial^{|q_2|} f}{\partial h_1^{q_{21}}\ldots\partial h_{K-1}^{q_{2(K-1)}}}(Y=s|X;a^*_j,b^*_j)\nonumber\\ 
    \label{eq:linear_combination}
    &\hspace{3cm}+\sum_{j=1}^{k_*}\sum_{|\gamma|=0}^{1+\mathbf{1}_{\{|\mathcal{C}_j|>1\}}}\widetilde{\eta}_{\gamma}(j)\times [M(X)]^{\gamma}\exp((\boj)^{\top}M(X))g_{G_*}(Y=s|X)=0.
\end{align}
\textbf{Regime 1:} For any $j\in[k_*]$, there exists an index $\ell\in[K-1]$ such that $b^*_{j\ell}\neq\zerod$.

By using the same arguments for proving the set $\mathcal{F}$ in equation~\eqref{eq:F_set} is linearly independent for almost surely $X$, we get that the following set also admits that property:
\begin{align*}
    &\Bigg\{[M(X)]^{q_4}X^{q_3}\exp((\boj)^{\top}M(X))\frac{\partial^{|q_2|} f}{\partial h_1^{q_{21}}\ldots\partial h_{K-1}^{q_{2(K-1)}}}(Y=s|X;a^*_j,b^*_j), \\
    &\hspace{.5cm}[M(X)]^{\gamma}\exp((\boj)^{\top}M(X))g_{G_*}(Y=s|X):j\in[k_*], \ 0\leq |q_2|+|q_4|, |\gamma|\leq 1+\mathbf{1}_{\{|\mathcal{C}_j|>1\}}, \ 0\leq |q_3|\leq |q_2|\Bigg\}.
\end{align*}
As a result, it follows that $\widetilde{\tau}_{q_2,q_3,q_4}(j)=\widetilde{\eta}_{\gamma}(j)=0$ for any $j\in[k_*]$, $0\leq |q_2|+|q_4|, |\gamma|\leq 1+\mathbf{1}_{\{|\mathcal{C}_j|>1\}}$ and $0\leq |q_3|\leq |q_2|$, which contradicts the fact that at least one among $\widetilde{\tau}_{q_2,q_3,q_4}(j)$ is different from zero. 

\textbf{Regime 2:} There exists an index $j\in[k_*]$ such that $b^*_{j\ell}=\zerod$ for any $\ell\in[K-1]$.

Note that equation~\eqref{eq:linear_combination} can be rewritten as
\begin{align*}
    \sum_{j=1}^{k_*}P^{(j)}(X)\exp((\beta^*_{1j})^{\top}M(X))+\sum_{j=1}^{k_*}Q^{(j)}(X)\exp((\beta^*_{1j})^{\top}M(X))g_{G_*}(Y=s|X)=0,
\end{align*}
where we define
\begin{align*}
    P^{(j)}(X)&:=\sum_{|q_4|+|q_2|=0}^{1+\mathbf{1}_{\{|\mathcal{C}_j|>1\}}}\sum_{|q_3|=0}^{|q_2|}\widetilde{\tau}_{q_2,q_3,q_4}(j)X^{q_3}[M(X)]^{q_4}\cdot\frac{\partial^{|q_2|}f}{\partial h_1^{q_{21}}\ldots\partial h_{K-1}^{q_{2(K-1)}}}(Y=s|X;a^*_j,b^*_j),\\
    Q^{(j)}(X)&:=\sum_{|\gamma|=0}^{1+\mathbf{1}_{\{|\mathcal{C}_j|>1\}}}\widetilde{\eta}_{\gamma}(j)[M(X)]^{\gamma}.
\end{align*}
Since the following set is linearly independent for almost surely $X$:
\begin{align*}
    \Big\{\exp((\beta^*_{1j})^{\top}M(X)), \ \exp((\beta^*_{1j})^{\top}M(X))g_{G_*}(Y=s|X):j\in[k_*]\Big\},
\end{align*}
we achieve that $P^{(j)}(X)=Q^{(j)}(X)=0$ for any $j\in[k_*]$ for almost surely $X$. Then, it follows from the formulation of $P^{(j)}(X)$ that
\begin{align*}
    \sum_{|q_4|+|q_2|=0}^{1+\mathbf{1}_{\{|\mathcal{C}_j|>1\}}}\sum_{|q_3|=0}^{|q_2|}\widetilde{\tau}_{q_2,q_3,q_4}(j)X^{q_3}[M(X)]^{q_4}=0,
\end{align*}
for any $j\in[k_*]$ for almost surely $X$. It can be seen from the above equation that $0\leq|q_3|+|q_4|\leq |q_2|+|q_4|\leq 1+\mathbf{1}_{\{|\mathcal{C}_j|>1\}}\leq 2$. Moreover, by definition of function $M$ (see Definition~\ref{def:modified_function}), the set $\Big\{X^p[M(X)]^{q}: p,q\in\mathbb{N}^d, \ 0\leq|p|+|q|\leq 2\Big\}$ is linearly independent for almost surely $X$. As a consequence, we achieve that $\widetilde{\tau}_{q_2,q_3,q_4}(j)=0$ for any $j\in[k_*]$, $0\leq |q_2|+|q_4|, |\gamma|\leq 1+\mathbf{1}_{\{|\mathcal{C}_j|>1\}}$ and $0\leq |q_3|\leq |q_2|$, contradicting the fact that at least one among $\widetilde{\tau}_{q_2,q_3,q_4}(j)$ is non-zero. 

Combine the results of the above two regimes, we reach the bound in equation~\eqref{eq:local_structure_modified}.
\section{Proofs for Convergence Rates of Density Estimation}
In this appendix, we present the proof of Propostion~\ref{prop:density_estimation} in Appendix~\ref{appendix:density_estimation}, while that for Proposition~\ref{prop:modified_density_estimation} is given in Appendix~\ref{appendix:modified_density_estimation}.
\label{appendix:general_density_estimation}
\subsection{Proof of Proposition~\ref{prop:density_estimation}}
\label{appendix:density_estimation}
In this appendix, we will firstly introduce key results on density estimation with MLE which are mainly based on \cite{Vandegeer-2000}, and then provide the proof of Proposition~\ref{prop:density_estimation} at the end.

\subsubsection{Key Results}
\label{appendix:key_results}
To begin with, it is necessary to define some notations that will be used in our presentation. First, we define $\mathcal{P}_k(\Theta):=\{g_{G}(Y|X):G\in\mathcal{O}_k(\Theta)\}$ as the set of conditional density functions of all mixing measures belonging to $\mathcal{O}_k(\Theta)$. In addition, let $N(\varepsilon,\mathcal{P}_k(\Theta),\|\cdot\|_{1})$ be the covering number \cite{Vandegeer-2000} of metric space $(\mathcal{P}_k(\Theta),\|\cdot\|_{1})$ while $H_B(\varepsilon,\mathcal{P}_k(\Theta),h)$ be the bracketing entropy \cite{Vandegeer-2000} of $\mathcal{P}_k(\Theta)$ under the Hellinger distance. Then, the following result gives us the upper bound of these quantities:
\begin{lemma}
    \label{lemma:empirical_processes}
    For any bounded set $\Theta$ and $\varepsilon\in(0,1/2)$, we have
    \begin{itemize}
        \item[(i)] $\log N(\varepsilon,\mathcal{P}_k(\Theta),\|\cdot\|_{1})\lesssim\log(1/\varepsilon)$;
        \item[(ii)] $H_B(\varepsilon,\mathcal{P}_k(\Theta),h)\lesssim \log(1/\varepsilon)$. 
    \end{itemize}
\end{lemma}
\begin{proof}[Proof of Lemma~\ref{lemma:empirical_processes}]
    \textbf{Part (i)} Firstly, we define $\Omega:=\{(a,b)\in\mathbb{R}^{K}\times\mathbb{R}^{d\times K}:(\beta_0,\beta_1,a,b)\in\Theta\}$. Since $\Theta$ is a compact set, then $\Omega$ is also compact. Therefore, $\Omega$ admits an $\varepsilon$-cover of size $T_2$ denoted by $\Omega_{\varepsilon}$. In addition, we also define $\Delta:=\{(\beta_0,\beta_1)\in\mathbb{R}\times\mathbb{R}^d:(\beta_0,\beta_1,a,b)\in\Theta\}$, and $\Delta_{\varepsilon}$ as an $\varepsilon$-cover of $\Delta$. It can be validated that
    \begin{align*}
        |\Omega_{\varepsilon}|\lesssim\mathcal{O}(\varepsilon^{-K(d+1)k}), \quad |\Delta_{\varepsilon}|\lesssim\mathcal{O}(\varepsilon^{-(d+1)k}).
    \end{align*}
    Next, given a mixing measure $G=\sum_{i=1}^{k'}\exp(\beta_{0i})\delta_{(\beta_{1i},a_i,b_i)}\in\mathcal{O}_k(\Theta)$, where $k'\in[k]$, we define $\widetilde{G}:=\sum_{i=1}^{k'}\exp(\beta_{0i})\delta_{(\beta_{1i},\overline{a}_{i},\overline{b}_{i})}$ in which $(\overline{a}_{i},\overline{b}_{i})\in\Omega_{\varepsilon}$ such that it is the closet point to $(a_{i},b_{i})$ for any $i\in[k']$. Additionally, we also consider the mixing measure $\overline{G}:=\sum_{i=1}^{k'}\exp(\overline{\beta}_{0i})\delta_{(\overline{\beta}_{1i},\overline{a}_{i},\overline{b}_{i})}$, where $(\overline{\beta}_{0i},\overline{\beta}_{1i})\in\Delta_{\varepsilon}$ is the closest point to $(\beta_{0i},\beta_{1i})$. By this construction, it can be justified that $g_{\overline{G}}\in\mathcal{H}$, where we define
    \begin{align*}
        \mathcal{H}:=\left\{g_{G}\in\mathcal{P}_k(\Theta):(\overline{\beta}_{0i},\overline{\beta}_{1i})\in\Delta_{\varepsilon}, \ (\overline{a}_{i},\overline{b}_{i})\in\Omega_{\varepsilon},\ \forall i\in[k]\right\}.
    \end{align*}
    Now, we show that $\mathcal{H}$ is an $\varepsilon$-cover of the metric space $(\mathcal{P}_k(\Omega),\|\cdot\|_{1})$ but not necessarily the smallest one. For that purpose, we aim to find a bound for the term $\|g_{G}-g_{\overline{G}}\|_1$. According to the triangle inequality, we have
    \begin{align*}
        \|g_{G}-g_{\overline{G}}\|_1\leq \|g_{G}-g_{\widetilde{G}}\|_1+\|g_{\widetilde{G}}-g_{\overline{G}}\|_1.
    \end{align*}
    Regarding the first term in the above right hand side, 
    \begin{align*}
        \|g_{\widetilde{G}}-g_{\overline{G}}\|_{1}&=\sum_{s=1}^{K}\int_{\mathcal{X}}\Big|g_{G}(Y|X)-g_{\widetilde{G}}(Y|X)\Big|\dint X\\
        &=\sum_{s=1}^{K}\int_{\mathcal{X}}\sum_{i=1}^{k'}\softmax(\beta_{1i}^{\top}X+\beta_{0i})\cdot \Big|f(Y=s|X;a_i,b_i)-f(Y=s|X;\overline{a}_i,\overline{b}_i)\Big|\dint X\\
        &\leq\sum_{s=1}^{K}\int_{\mathcal{X}} \sum_{i=1}^{k'}\Big|\softmax(a_{is}+b_{is}^{\top}X)-\softmax(\overline{a}_{is}+\overline{b}_{is}^{\top}X)\Big|\dint X.
    \end{align*}
    Since $\softmax$ is a differentiable function, it is also a $L$-Lipschitz function where $L>0$. Additionally, as $\mathcal{X}$ is a bounded function, we may assume that $\|X\|\leq B$ for some constant $B>0$. As a result, we have
    \begin{align*}
        \|g_{\widetilde{G}}-g_{\overline{G}}\|_{1}&\leq\sum_{s=1}^{K}\sum_{i=1}^{k'} L\cdot\Big(|a_{is}-\overline{a}_{is}|+\|X\|\cdot\|b_{is}-\overline{b}_{is}\|\Big)\\
        &\leq Kk'L\cdot(\varepsilon+B\varepsilon)\lesssim\varepsilon.
    \end{align*}
    Similarly, the second term is bounded as
    \begin{align*}
        \|g_{\widetilde{G}}-g_{\overline{G}}\|_{1}&= \sum_{s=1}^{K}\int_{\mathcal{X}}\sum_{i=1}^{k'}\Big|\softmax(\beta_{1i}^{\top}X+\beta_{0i})-\softmax(\overline{\beta}_{1i}^{\top}X+\overline{\beta}_{0i})\Big|\cdot f(Y|X;\overline{a}_{is},\overline{b}_{is})\dint X\\
        &\leq\sum_{s=1}^{K}\sum_{i=1}^{k'}L\cdot\Big(\|\beta_{1i}-\overline{\beta}_{1i}\|\cdot\|X\|+|\beta_{0i}-\overline{\beta}_{0i}|\Big)\\
        &\leq Kk'L(B\varepsilon+\varepsilon)\lesssim\varepsilon.
    \end{align*}
    Consequently, we get $\|g_{G}-g_{\overline{G}}\|_1\lesssim\varepsilon$.
    This result implies that $\mathcal{H}$ is an $\varepsilon$-cover of the metric space $(\mathcal{P}_k(\Theta),\|\cdot\|_1)$. Then, it follows from the definition of the covering number that 
    \begin{align*}
        N\left(\varepsilon,\mathcal{P}_k(\Theta),\|\cdot\|_{1}\right)\leq|\mathcal{H}|=|\Theta_{\varepsilon}|\times |\Delta_{\varepsilon}|= \mathcal{O}(\varepsilon^{-K(d+1)k})\times\mathcal{O}(\varepsilon^{-(d+1)k})=\mathcal{O}(\varepsilon^{-(K+1)(d+1)k}),
    \end{align*}
    which is equivalent to $\log N\left(\varepsilon,\mathcal{P}_k(\Theta),\|\cdot\|_{1}\right)\leq \log(1/\varepsilon)$. 

    \textbf{Part (ii)} Given $\varepsilon>0$ and let $\eta\leq\varepsilon$ that we will choose later. Assume that $\mathcal{P}_k(\Theta)$ has an $\eta$-cover denoted by $\{p_1,p_2,\ldots,p_N\}$ where $N:=N(\eta,\mathcal{P}_k(\Theta),\|\cdot\|_1)$. Next, we start to construct brackets of the form $[L_i(Y|X),U_i(Y|X)]$ for all $i\in[N]$ as below:
    \begin{align*}
        L_i(Y|X)&:=\max\{p_i(Y|X)-\eta,0\},\\
        U_i(Y|X)&:=\max\{p_i(Y|X)+\eta,1\}.
    \end{align*}
    By this construction, we can verify that $\mathcal{P}_k(\Theta)\subset\bigcup_{i=1}^N[L_i(Y|X),U_i(Y|X)]$ and $U_i(Y|X)-L_i(Y|X)\leq \min\{2\eta,1\}$. Additionally, we also have
    \begin{align*}
        \|U_i(\cdot|X)-L_i(\cdot|X)\|_1=\sum_{\ell=1}^{K}[U_i(Y=\ell|X)-L_i(Y=\ell|X)]\leq 2K\eta.
    \end{align*}
    By definition, since $H_B(2K\eta,\mathcal{P}_k(\Theta),\|\cdot\|_1)$ is the logarithm of the smallest number of brackets of size $2K\eta$ required for covering $\mathcal{P}_k(\Theta)$, we obtain that
    \begin{align*}
        H_B(2K\eta,\mathcal{P}_k(\Theta),\|\cdot\|_1)\leq \log N(\eta,\mathcal{P}_k(\Theta),\|\cdot\|_1)\leq\log(1/\eta),
    \end{align*}
    where the last inequality follows from the result of Part (i). Thus, by choosing $\eta=\varepsilon/(2K)$, we have $H_B(\varepsilon,\mathcal{P}_k(\Theta),\|\cdot\|_1)\lesssim\log(1/\varepsilon)$. Furthermore, as $h\leq \|\cdot\|_1$, we achieve the desired conclusion:
    \begin{align*}
        H_B(\varepsilon,\mathcal{P}_k(\Theta),h)\lesssim\log(1/\varepsilon).
    \end{align*}
    Hence, the proof is completed.
\end{proof}
Subsequently, we denote 
\begin{align*}
   \overline{\mathcal{P}}_k(\Theta)&:=\{g_{(G+G_*)/2}(Y|X):G\in\mathcal{O}_k(\Theta)\},\\
   \overline{\mathcal{P}}^{1/2}_k(\Theta)&:=\{g^{1/2}_{(G+G_*)/2}(Y|X):G\in\mathcal{O}_k(\Theta)\}.
\end{align*}
For any $\xi>0$, the Hellinger ball centered around the true conditional density $g_{G_*}(Y|X)$ and intersected with the set $\overline{\mathcal{P}}^{1/2}_k(\Theta)$ is defined as
\begin{align*}
    \overline{\mathcal{P}}^{1/2}_k(\Theta,\xi):=\{g^{1/2}\in\overline{\mathcal{P}}_k(\Theta):\bbE_X[h(g(\cdot|X),g_{G_*}(\cdot|X))]\leq \xi\}.
\end{align*}
Moreover, Geer et al. \cite{Vandegeer-2000} proposes the following term to capture the size of the Hellinger ball $\overline{\mathcal{P}}^{1/2}_k(\Theta,\xi)$:
\begin{align}
    \label{eq:Hellinger_ball_size}
    \mathcal{J}_B(\xi,\overline{\mathcal{P}}^{1/2}_k(\Theta,\xi)):=\int_{\xi^2/2^{13}}^{\xi}H_B^{1/2}(t,\overline{\mathcal{P}}^{1/2}_k(\Theta,t),\|\cdot\|)\dint t \vee \xi,    
\end{align}
where $t\vee \xi:=\max\{t,\xi\}$. Now, let us recall below an important result regarding the density estimation rate in \cite{Vandegeer-2000} with adapted notations of this paper. 
\begin{lemma}[Theorem 7.4, \cite{Vandegeer-2000}]
    \label{lemma:vandegeer}
    Take $\Psi(\xi)\geq \mathcal{J}_B(\xi,\overline{\mathcal{P}}^{1/2}_k(\Theta,\xi))$ such that $\Psi(\xi)/\xi^2$ is a non-increasing function of $\xi$. Then, given a universal constant $c$ and a sequence $(\xi_n)$ that satisfies $\sqrt{n}\xi^2_n\geq c\Psi(\xi_n)$, we get
    \begin{align*}
        \bbP\Big(\bbE_X[h(g_{\widehat{G}_n}(\cdot|X),g_{G_*}(\cdot|X))]>\xi\Big)\leq c\exp\Big(-\frac{n\xi^2}{c^2}\Big),
    \end{align*}
    for any $\xi\geq\xi_n$.
\end{lemma}
The proof of this Lemma is in \cite{Vandegeer-2000}. 

\subsubsection{Main Proof}
\label{appendix:proof_density_estimation}

It is worth noting that $H_B(t,\overline{\mathcal{P}}^{1/2}_k(\Theta,t),\|\cdot\|)\leq H_B(t,\mathcal{P}_k(\Theta),h)$ for any $t>0$. Then, equation~\eqref{eq:Hellinger_ball_size} indicates that
\begin{align*}
    \mathcal{J}_B(\xi,\overline{\mathcal{P}}^{1/2}_k(\Theta,\xi))\leq \int_{\xi^2/2^{13}}^{\xi}H_B^{1/2}(t,\overline{\mathcal{P}}^{1/2}_k(\Theta,t),h)\dint t\vee\xi \lesssim \int_{\xi^2/2^{13}}^{\xi}\log(1/t)\dint t\vee\xi,
\end{align*}
where the second inequality follows from part (ii) of Lemma~\ref{lemma:empirical_processes}. By setting $\Psi(\xi)=\xi\sqrt{\log(1/\xi)}$ such that $\Psi(\xi)\geq \mathcal{J}_B(\xi,\overline{\mathcal{P}}^{1/2}_k(\Theta,\xi))$ and $\xi_n=\xi\sqrt{\log(1/\xi)}$, Lemma~\ref{lemma:vandegeer} gives us that
\begin{align*}
    \bbP\Big(\bbE_X[h(g_{\widehat{G}_n}(\cdot|X),g_{G_*}(\cdot|X))]>\xi\Big)\leq c\exp\Big(-\frac{n\xi^2}{c^2}\Big),
\end{align*}
where $C$ and $c$ are universal positive constants depending only on $\Theta$. Hence, the proof is completed.

\subsection{Proof of Proposition~\ref{prop:modified_density_estimation}}
\label{appendix:modified_density_estimation}
From Definition~\ref{def:modified_function}, since $M(X)$ is a bounded function of $X$, the arguments presented in Appendix~\ref{appendix:density_estimation} still hold under the modified softmax gating multinomial logistic mixture of experts.
\section{Proofs for the Identifiablity of the (Modified) Softmax Gating Multinomial Logistic MoE}
\label{appendix:general_identifiability}
In this appendix, we provide the proofs of Proposition~\ref{prop:identifiability} and Proposition~\ref{prop:modified_identifiability} in Appendix~\ref{appendix:identifiability} and Appendix~\ref{appendix:modified_identifiability}, respectively.
\subsection{Proof of Proposition~\ref{prop:identifiability}}
\label{appendix:identifiability}
From the assumption of Proposition~\ref{prop:identifiability}, the following equation holds for any $s\in[K]$ and almost surely $X\in\mathcal{X}$:
\begin{align}
    \sum_{i=1}^{k}\frac{\exp((\beta_{1i})^{\top}X+\beta_{0i})}{\sum_{j=1}^{k}\exp((\beta_{1i})^{\top}X+\beta_{0i})}\cdot&\frac{\exp(a_{is}+(b_{is})^{\top}X)}{\sum_{\ell=1}^{K}\exp(a_{i\ell}+(b_{i\ell})^{\top}X)}\nonumber\\
    \label{eq:identifiable_equation}
    &=\sum_{i=1}^{k'}\frac{\exp((\beta'_{1i})^{\top}X+\beta'_{0i})}{\sum_{j=1}^{k'}\exp((\beta'_{1i})^{\top}X+\beta'_{0i})}\cdot\frac{\exp(a'_{is}+(b'_{is})^{\top}X)}{\sum_{\ell=1}^{K}\exp(a'_{i\ell}+(b'_{i\ell})^{\top}X)}.
\end{align}
According to \cite{grun2008iden}, the multinomial logistic mixtures are identifiable, which implies that two mixing measures $G$ and $G'$ share the number of experts and the gating set of the mixing measure, i.e. $k=k'$ and
\begin{align*}
    \left\{\frac{\exp((\beta_{1i})^{\top}X+\beta_{0i})}{\sum_{j=1}^{k}\exp((\beta_{1i})^{\top}X+\beta_{0i})}:i\in[k]\right\}\equiv\left\{\frac{\exp((\beta'_{1i})^{\top}X+\beta'_{0i})}{\sum_{j=1}^{k}\exp((\beta'_{1i})^{\top}X+\beta'_{0i})}:i\in[k]\right\},
\end{align*}
for almost surely $X\in\mathcal{X}$. WLOG, we assume that
\begin{align*}
    \frac{\exp((\beta_{1i})^{\top}X+\beta_{0i})}{\sum_{j=1}^{k}\exp((\beta_{1i})^{\top}X+\beta_{0i})}=\frac{\exp((\beta'_{1i})^{\top}X+\beta'_{0i})}{\sum_{j=1}^{k}\exp((\beta'_{1i})^{\top}X+\beta'_{0i})},
\end{align*}
for any $i\in[k]$. As $\beta_{1k}=\beta'_{1k}=\zerod$ and $\beta_{0k}=\beta'_{0k}=0$, the above result leads to $\beta_{1i}=\beta'_{1i}$ and $\beta_{0i}=\beta'_{0i}$ for all $i\in[k]$. Thus, the equation~\eqref{eq:identifiable_equation} becomes
\begin{align}
    \label{eq:simplified_equation}
    \sum_{i=1}^{k}\exp(\beta_{0i})u(Y=s|X;\beta_{1i},a_i,b_i)=\sum_{i=1}^{k}\exp(\beta_{0i})u(Y=s|X;\beta_{1i},a'_i,b'_i),
\end{align}
for any $s\in[K]$ and almost surely $X\in\mathcal{X}$, where $u(Y=s|X;\beta_{1i},a_i,b_i):=\exp(\beta_{1i}^{\top}X)\cdot\dfrac{\exp(a_{is}+(b_{is})^{\top}X)}{\sum_{\ell=1}^{K}\exp(a_{i\ell}+(b_{i\ell})^{\top}X)}$ and $a_i=(a_{i1},a_{i2},\ldots,a_{iK})$, $b_i=(b_{i1},b_{i2},\ldots,b_{iK})$. 

Subsequently, we will consider $r$ subsets of the set $[k]$, denoted by $S_1,S_2,\ldots,S_r$ that satisfy the following property: $\exp(\beta_{0i})=\exp(\beta_{0i'})$ for any $i,i'\in S_{t}$ for $t\in[r]$. Therefore, we can rewrite equation~\eqref{eq:simplified_equation} as
\begin{align*}
    \sum_{t=1}^{r}\sum_{i\in S_{t}}\exp(\beta_{0i})u(Y=s|X;\beta_{1i},a_i,b_i)=\sum_{t=1}^{r}\sum_{i\in S_{t}}\exp(\beta_{0i})u(Y=s|X;\beta_{1i},a'_i,b'_i),
\end{align*}
for any $s\in[K]$ and almost surely $X\in\mathcal{X}$. It follows from the above equation that for each $t\in[r]$, we get $\{(a_{i\ell}+(b_{i\ell})^{\top}X)_{\ell=1}^{K}:i\in S_t\}\equiv\{(a'_{i\ell}+(b'_{i\ell})^{\top}X)_{\ell=1}^{K}:i\in S_t\}$ for almost surely $X\in\mathcal{X}$. This leads to
\begin{align*}
    \Big\{(a_{i1},\ldots,a_{iK},b_{i1},\ldots,b_{iK}):i\in S_t\Big\}\equiv \Big\{(a'_{i1},\ldots,a'_{iK},b'_{i1},\ldots,b'_{iK}):i\in S_t\Big\}.
\end{align*}
Again, we may assume WLOG that $(a_{i1},\ldots,a_{iK},b_{i1},\ldots,b_{iK})=(a'_{i1},\ldots,a'_{iK},b'_{i1},\ldots,b'_{iK})$ for any $i\in S_t$. As a result, we obtain that
\begin{align*}
    \sum_{t=1}^{r}\sum_{i\in S_t}\exp(\beta_{0i})\delta_{(\beta_{1i},a_{i1},\ldots,a_{iK},b_{i1},\ldots,b_{iK})}=\sum_{t=1}^{r}\sum_{i\in S_t}\exp(\beta'_{0i})\delta_{(\beta'_{1i},a'_{i1},\ldots,a'_{iK},b'_{i1},\ldots,b'_{iK})}.
\end{align*}
In other words, we achieve that $G\equiv G'$, which completes the proof.
\subsection{Proof of Proposition~\ref{prop:modified_identifiability}}
\label{appendix:modified_identifiability}
According to the assumption of Proposition~\ref{prop:modified_identifiability}, the following equation holds for any $s\in[K]$ and almost surely $X\in\mathcal{X}$:
\begin{align}
    \sum_{i=1}^{k}\frac{\exp((\beta_{1i})^{\top}M(X)+\beta_{0i})}{\sum_{j=1}^{k}\exp((\beta_{1i})^{\top}M(X)+\beta_{0i})}\cdot&\frac{\exp(a_{is}+(b_{is})^{\top}X)}{\sum_{\ell=1}^{K}\exp(a_{i\ell}+(b_{i\ell})^{\top}X)}\nonumber\\
    \label{eq:modified_identifiable_equation}
    &=\sum_{i=1}^{k'}\frac{\exp((\beta'_{1i})^{\top}M(X)+\beta'_{0i})}{\sum_{j=1}^{k'}\exp((\beta'_{1i})^{\top}M(X)+\beta'_{0i})}\cdot\frac{\exp(a'_{is}+(b'_{is})^{\top}X)}{\sum_{\ell=1}^{K}\exp(a'_{i\ell}+(b'_{i\ell})^{\top}X)}.
\end{align}
Since the multinomial logistic mixtures are identifiable (see \cite{grun2008iden}), two mixing measures $G$ and $G'$ admit the same number of experts and the same gating set, i.e. $k=k'$ and
\begin{align*}
    \left\{\frac{\exp((\beta_{1i})^{\top}M(X)+\beta_{0i})}{\sum_{j=1}^{k}\exp((\beta_{1i})^{\top}M(X)+\beta_{0i})}:i\in[k]\right\}\equiv\left\{\frac{\exp((\beta'_{1i})^{\top}M(X)+\beta'_{0i})}{\sum_{j=1}^{k}\exp((\beta'_{1i})^{\top}M(X)+\beta'_{0i})}:i\in[k]\right\},
\end{align*}
for almost surely $X\in\mathcal{X}$. WLOG, we assume that
\begin{align*}
    \frac{\exp((\beta_{1i})^{\top}X+\beta_{0i})}{\sum_{j=1}^{k}\exp((\beta_{1i})^{\top}X+\beta_{0i})}=\frac{\exp((\beta'_{1i})^{\top}X+\beta'_{0i})}{\sum_{j=1}^{k}\exp((\beta'_{1i})^{\top}X+\beta'_{0i})},
\end{align*}
for any $i\in[k]$. From the Definition~\ref{def:modified_function}, we know that $M(X)$ is a bounded function of $X$. Moreover, since $\beta_{1k}=\beta'_{1k}=\zerod$ and $\beta_{0k}=\beta'_{0k}=0$, the above result implies that $\beta_{1i}=\beta'_{1i}$ and $\beta_{0i}=\beta'_{0i}$ for all $i\in[k]$. Therefore, the equation~\eqref{eq:modified_identifiable_equation} can be reformulated as follows:
\begin{align*}
    \sum_{i=1}^{k}\exp(\beta_{0i})u(Y=s|X;\beta_{1i},a_i,b_i)=\sum_{i=1}^{k}\exp(\beta_{0i})u(Y=s|X;\beta_{1i},a'_i,b'_i),
\end{align*}
for any $s\in[K]$ and almost surely $X\in\mathcal{X}$, where $u(Y=s|X;\beta_{1i},a_i,b_i):=\exp(\beta_{1i}^{\top}X)\cdot\dfrac{\exp(a_{is}+(b_{is})^{\top}X)}{\sum_{\ell=1}^{K}\exp(a_{i\ell}+(b_{i\ell})^{\top}X)}$ and $a_i=(a_{i1},a_{i2},\ldots,a_{iK})$, $b_i=(b_{i1},b_{i2},\ldots,b_{iK})$. Then, we can apply the arguments used in Appendix~\ref{appendix:identifiability} to deduce that $G\equiv G'$.

\section{Simulation Studies}
\label{appendix:simulation}
In this appendix, we carry out several numerical experiments to empirically verify our theoretical results regarding the convergence rates of maximum likelihood estimation in the standard softmax gating multinomial logistic MoE model under the Regime 1 in Appendix~\ref{appendix:regime_1}. Meanwhile, under the Regime 2, we aim to empirically demonstrate the benefits of using modified softmax gating functions over the standard softmax gating function in the parameter estimation problem of the multinomial logistic MoE model in Appendix~\ref{appendix:regime_2}.
\subsection{Regime 1}
\label{appendix:regime_1}
%
\textbf{Synthetic Data.} We first sample the covariates $X$ from the uniform distribution over $[0,1]$. Then, we draw the response $Y$ from the following conditional density $g_{G_{*}}(Y = s | X)$ of a softmax gating binomial logistic mixture of $k_*=2$ experts: 
\begin{align}
    \label{eq:new_density}
    g_{G_{*}}(Y = s | X):=\sum_{i=1}^{2}~\frac{\exp((\beta^*_{1i})^{\top}X+\beta^*_{0i})}{\sum_{j=1}^{2}\exp((\boj)^{\top}X+\bzj)}\times\frac{\exp(a^*_{is}+(b^*_{is})^{\top}X)}{\sum_{\ell=1}^{K}\exp(a^*_{i\ell}+(b^*_{i\ell})^{\top}X)},
\end{align}
for $s\in[K]$, where $K=2$. Here, the true mixing measure $G_{*} = \sum_{i = 1}^{2} \exp(\beta_{0i}^{*}) \delta_{(\beta_{1i}^{*}, a_{i1}^{*},a_{i2}^{*}, b_{i1}^{*}, b_{i2}^{*})}$ consists of $k_*=2$ components with parameters given in Table \ref{tab:true-regime1}, which satisfy the assumptions of Regime 1.
\begin{table}[ht]
    \centering
    \begin{tabular}{c c  c c}
    \toprule
                    & Gating parameters & 
                    \multicolumn{2}{c}{Expert parameters}\\
                {}  & {} & {\small Class 1} & {\small Class 2}\\
    \midrule
        $i=1$  & $(\beta_{01}^{*}, \beta_{11}^{*}) = (1, 3)$ & 
        
                   $(a_{11}^{*}, b_{11}^{*}) = (-1, 2)$ & $(a_{12}^{*}, b_{12}^{*}) = (0, 0)$\\

        $i=2$ & $(\beta_{02}^{*}, \beta_{12}^{*}) = (0, 0)$ & $ (a_{21}^{*}, b_{21}^{*}) = (1, -1)$ & $(a_{22}^{*}, b_{22}^{*}) = (0, 0)$\\
    \bottomrule
    \end{tabular}
    \caption{True parameters under the Regime 1.}
    \label{tab:true-regime1}
\end{table}


\textbf{Initialization.}  We then compute the MLE $\widehat{G}_n$ \wrt with the number of components $k \in \left\{k_*+1,k_*+2\right\}$ for each sample using the EM algorithm \cite{dempster_maximum_1977} with convergence criterion $\varepsilon=10^{-6}$ and $2000$ maximum EM iterations. 
For each $k\in\{k_*+1,k_*+2\}$, we randomly assign elements of the set $\{1, 2, ..., k\}$ into $k_*$ distinct sets $S_1, S_2, \ldots, S_{k_*}$, ensuring that each set contains at least one element. Moreover, we repeat this process for each replication. Following this, for each $i\in[k_*]$, we initialize the parameters by sampling from a Gaussian distribution with a mean centered around its true counterpart and a small variance. For each of $200$ different choices of sample size $n$ between $n_{\min}=10^4$ and $n_{\max} = 10^5$, we generate $40$ samples of size $n$.
All the code for our simulation study was written in Python 3.9 on a standard Unix machine.  
%
%

\textbf{Empirical Convergence Rates.} Subsequently, we report the empirical means of the discrepancy $\mathcal{D}_2$ between $\widehat{G}_n$ and $G_*$, and the choices of $k$ under the Regime 1 in Figure~\ref{fig_plot_model_regime1}.
%
It can be observed from Figures \ref{fig_plot_model_regime1_k3} and \ref{fig_plot_model_regime1_over_fitted_k4} that the empirical vanishing rates of the average discrepancy $\mathcal{D}_2(\widehat{G}_n,G_*)$ are of orders $\widetilde{\mathcal{O}}(n^{-0.48})$ and $\widetilde{\mathcal{O}}(n^{-0.43})$ when $k=3$ and $k=4$, respectively. These rates are slightly slower than the theoretical rate of order $\widetilde{\mathcal{O}}(n^{-0.5})$ in Theorem~\ref{theorem:standard_one}.
The main reason is that there has been only theoretical guarantee of global convergence for the parameter estimation under the mixture of experts with covariate-free gating function (see \cite{kwon2021minimax,kwon2020em,kwon2019global}), while that for the softmax gating mixture of experts has remained missing in the literature. 
In order for the empirical vanishing rate to match the theoretical one, the sample size $n$ must be large enough to compensate for the global convergence problem.

%
\begin{figure}[!ht]
    \centering
    \begin{subfigure}{.47\textwidth}
    \centering
    \includegraphics[scale = .49]{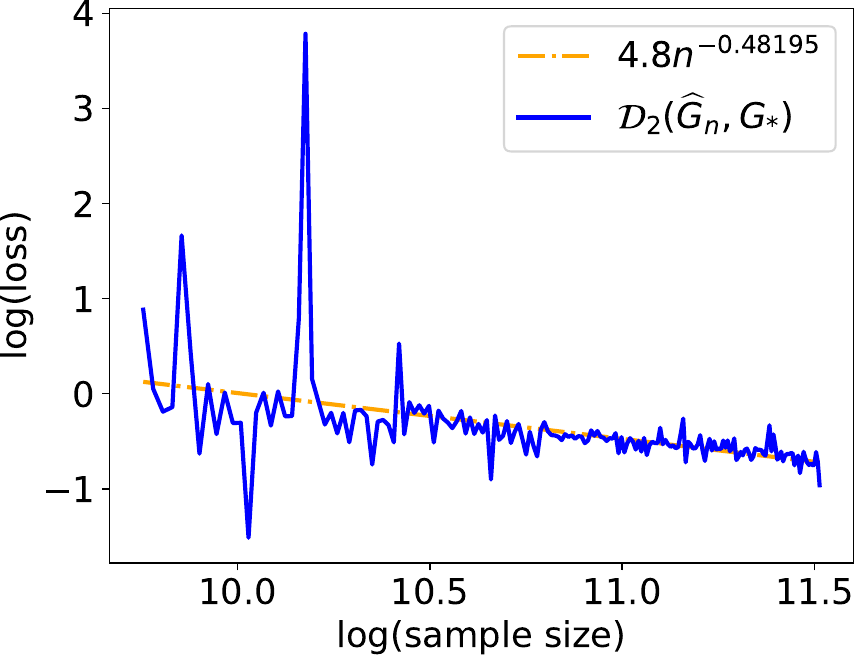}
    \caption{Over-specified setting with $k=3$ experts
    }
    \label{fig_plot_model_regime1_k3}
\end{subfigure}
    \hspace{0.6cm}
        \begin{subfigure}{.47\textwidth}
    \centering
    \includegraphics[scale = .49]{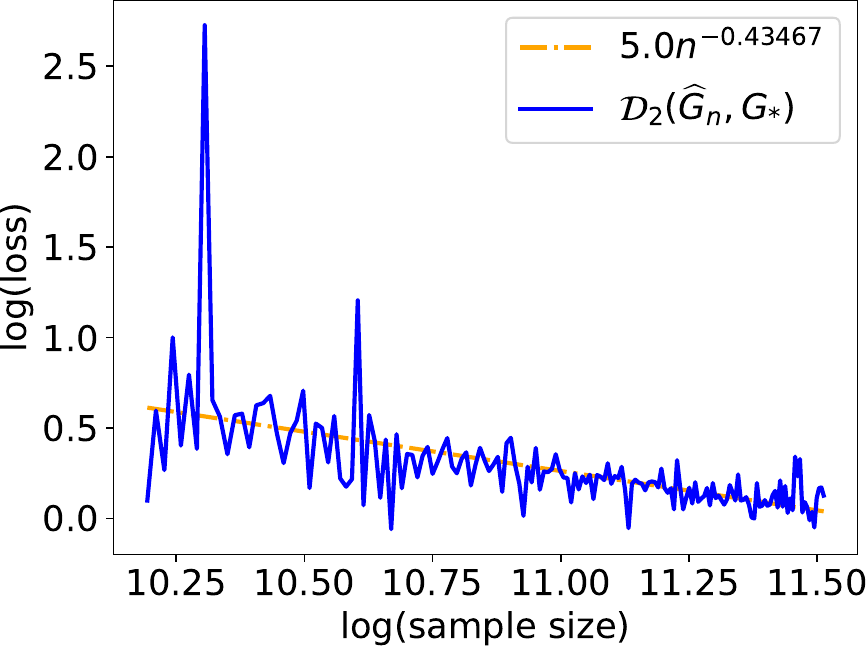}
    \caption{Over-specified setting with $k=4$ experts
    }	
    \label{fig_plot_model_regime1_over_fitted_k4}
\end{subfigure}
    %
    %
    \caption{Two log-log scaled plots for the empirical convergence rates of the MLE $\widehat{G}_n$ when the true model in equation~\eqref{eq:new_density} is over-specified by a softmax gating binomial logistic mixture of $k=3$ and $k=4$ experts, respectively.
    In these figures, the empirical means of the discrepancy $\mathcal{D}_2(\widehat{G}_n,G_*)$ are illustrated by the blue curves, while the oranges dash-dotted lines represent for the least-squares fitted linear regression lines. \label{fig_plot_model_regime1}}
\end{figure}

\subsection{Regime 2}
\label{appendix:regime_2}
\textbf{Synthetic Data.} We also generate the covariates $X$ from the uniform distribution over $[0,1]$. For the multinomial logistic MoE model with standard softmax gating function, we draw the response $Y$ from the conditional density $g_{G_*}(Y=s|X)$ given in equation~\eqref{eq:new_density}, while for the modified softmax gating function, we sample $Y$ from the following conditional density:
\begin{align}
    \widetilde{g}_{G_*}(Y=s|X):=\sum_{i=1}^{2}~\frac{\exp((\beta^*_{1i})^{\top}M(X)+\beta^*_{0i})}{\sum_{j=1}^{2}\exp((\boj)^{\top}M(X)+\bzj)}\times\frac{\exp(a^*_{is}+(b^*_{is})^{\top}X)}{\sum_{\ell=1}^{K}\exp(a^*_{i\ell}+(b^*_{i\ell})^{\top}X)},\label{eq:new_density2}
\end{align}
where we consider the standard softmax gating function $M(X) = X$ and $M(X) = \mathrm{sigmoid}(X)$. Next, while we keep the formulation of the true mixing measure $G_*=\sum_{i = 1}^{2} \exp(\beta_{0i}^{*}) \delta_{(\beta_{1i}^{*}, a_{i1}^{*},a_{i2}^{*}, b_{i1}^{*}, b_{i2}^{*})}$, the parameter values are slightly changed to satisfy the assumptions of Regime 2. In particular, we set $b^*_{21}=b^*_{22}=0$, while other parameters remains the same. More details can be found in Table~\ref{tab:true-regime2}.
\begin{table}[ht!]
    \centering
    \begin{tabular}{c c  c c}
    \toprule
                    & Gating parameters & 
                    \multicolumn{2}{c}{Expert parameters}\\
                {}  & {} & {\small Class 1} & {\small Class 2}\\
    \midrule
        $i=1$  & $(\beta_{01}^{*}, \beta_{11}^{*}) = (1, 3)$ & 
        
                   $(a_{11}^{*}, b_{11}^{*}) = (-1, 2)$ & $(a_{12}^{*}, b_{12}^{*}) = (0, 0)$\\

        $i=2$ & $(\beta_{02}^{*}, \beta_{12}^{*}) = (0, 0)$ & $ (a_{21}^{*}, b_{21}^{*}) = (1, 0)$ & $(a_{22}^{*}, b_{22}^{*}) = (0, 0)$\\
    \bottomrule
    \end{tabular}
    \caption{True parameters under the Regime 2.}
    \label{tab:true-regime2}
\end{table}

\subsubsection{Empirical Convergence Rates of the Voronoi-based Loss}
\begin{figure}[!ht]
    \centering
    \begin{subfigure}{.47\textwidth}
    \centering
    \includegraphics[scale = .46]{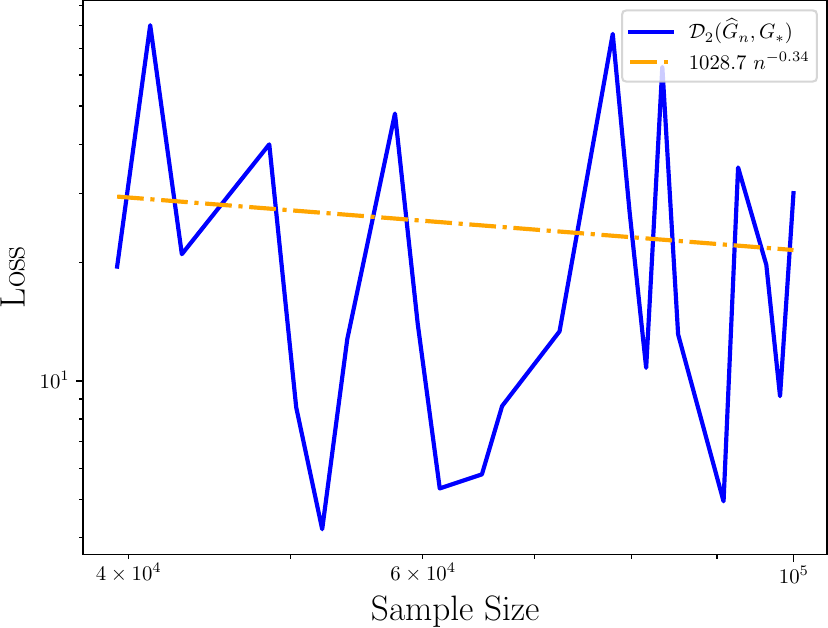}
    \caption{Over-specified setting with $k=3$ experts,\\ $M(X) = X$
    }
    \label{fig_plot_model_regime2_k3}
\end{subfigure}
    \hfill
\begin{subfigure}{.47\textwidth}
    \centering
    \includegraphics[scale = .46]{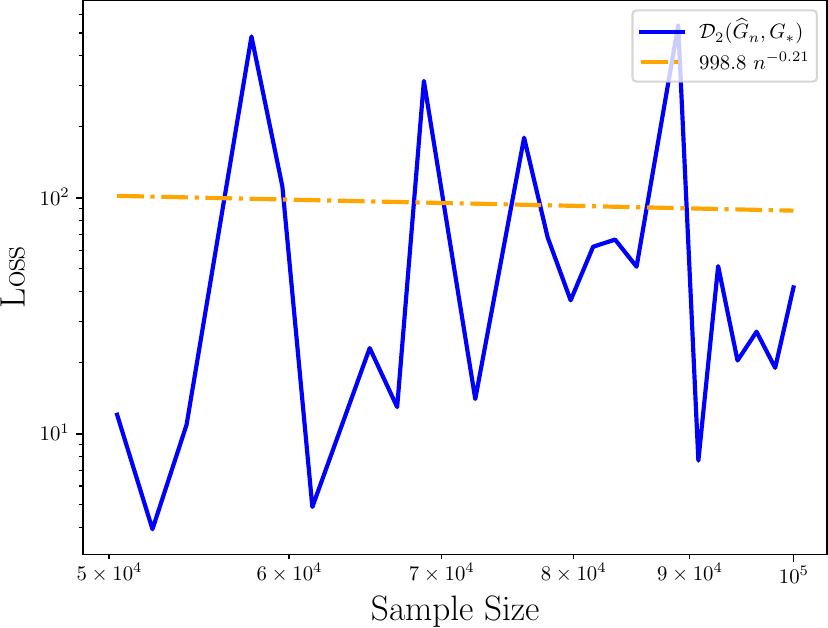}
    \caption{Over-specified setting with $k=4$ experts,\\ $M(X) = X$
    }
    \label{fig_plot_model_regime2_over_fitted_k4}
\end{subfigure}
\vskip\baselineskip
\begin{subfigure}{.47\textwidth}
    \centering
    \includegraphics[scale = .49]{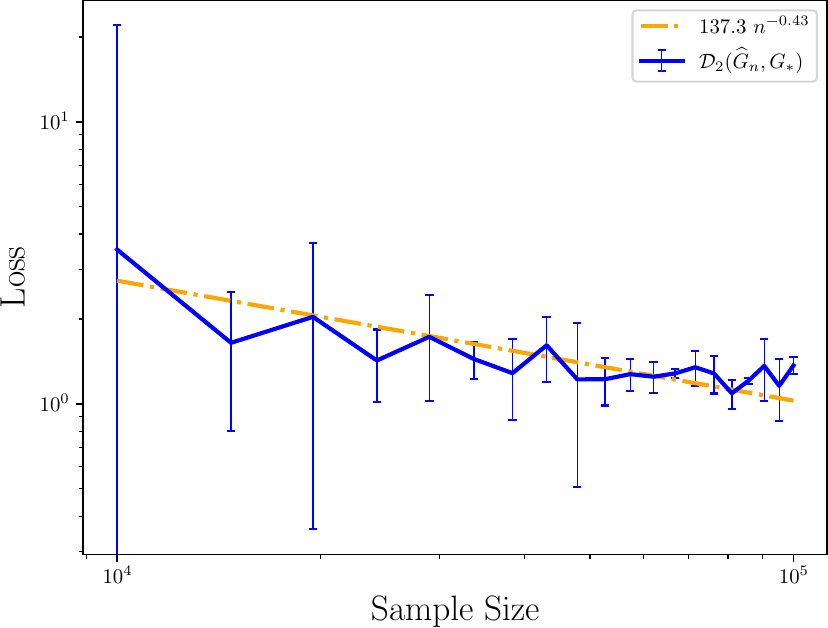}
    \caption{Over-specified setting with $k=3$ experts, $M(X) = \mathrm{sigmoid}(X)$.
    }
    \label{fig:regime2-k3-sigmoid}
\end{subfigure}
\hfill
\begin{subfigure}{.47\textwidth}
    \centering
    \includegraphics[scale = .49]{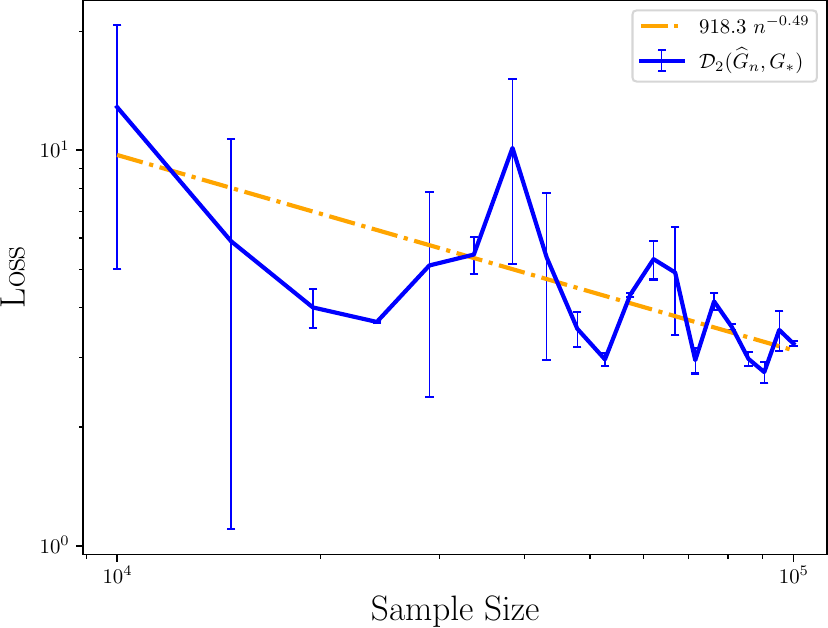}
    \caption{Over-specified setting with $k=4$ experts, $M(X) = \mathrm{sigmoid}(X)$.
    }	
    \label{fig:regime2-k4-sigmoid}
\end{subfigure}
    \caption{log-log scaled plots for the empirical convergence rates of the MLE $\widehat{G}_n$ when the true model in equation~\eqref{eq:new_density2} is over-specified by a softmax gating binomial logistic mixture with $M(X) = X$ and $M(X) = \mathrm{sigmoid}(X)$ of $k=3$ and $k=4$ experts, respectively.
    In these figures, the empirical means of the discrepancy $\mathcal{D}_2(\widehat{G}_n,G_*)$ are illustrated by the blue curves, while the oranges dash-dotted lines represent for the least-squares fitted linear regression lines. \label{fig_plot_model_regime2}}
\end{figure}

\textbf{Initialization.}
We first initialize fitted parameters in the same fashion as those in Appendix~\ref{appendix:regime_1}.

\textbf{Standard Gating Function.}
For $k=3$ we choose $35$ different values of sample size $n$ between $n_{\min}\approx 37\times10^3$ and $n_{\max} = 10^5$, while for $k = 4$ we select $28$ different choices of sample size $n$ between $n_{\min}\approx 50\times10^3$ and $n_{\max} = 10^5$. In both cases we generate the corresponding samples of size $n$.

\textbf{Modified Gating Function.}
We employ $M(X) = \mathrm{sigmoid}(X)$ and conduct 40 experiments for each sample size, covering a spectrum of 20 different sample sizes ranging from $10^4$ to $10^5$.

\textbf{Empirical Convergence Rates.} Subsequently, we report the empirical means of the discrepancy $\mathcal{D}_2$ between $\widehat{G}_n$ and $G_*$, and the choices of $k$ under the Regime 2 for standard gating function and modified gating function in Figure~\ref{fig_plot_model_regime2}.
It can be observed from Figures \ref{fig_plot_model_regime2_k3} and \ref{fig_plot_model_regime2_over_fitted_k4} that the empirical vanishing rates of the average discrepancy $\mathcal{D}_2(\widehat{G}_n,G_*)$ are of orders $\widetilde{\mathcal{O}}(n^{-0.34})$ and $\widetilde{\mathcal{O}}(n^{-0.21})$ when $k=3$ and $k=4$, respectively. These rates are slightly slower than the theoretical rate of order $\widetilde{\mathcal{O}}(n^{-1/2r})$ for some $ r\geq 1$ in Theorem~\ref{theorem:standard_two}.
Moreover, for $M(X) = \mathrm{sigmoid}(X)$ as illustrated in Figures~\ref{fig:regime2-k3-sigmoid} and \ref{fig:regime2-k4-sigmoid} the utilization of $M(X) = \mathrm{sigmoid}(X)$ in the gating function resulted in an enhanced convergence rate of $\widetilde{\mathcal{O}}(n^{-1/2})$ for both $k=3$ and $k=4$.
The main reasons are that the literature lacks a theoretical guarantee of global convergence for parameter estimation under the softmax mixture of experts, and that the MLE $\widehat{G}_n$ (which is the standard softmax mixture function) takes a very long time to converge to the true mixture measure $G_*$ even when we use $2000$ maximum number of EM iterations, see more in Appendices~\ref{appendix:regime_1} and ~\ref{sec_ECEM}. 


\subsubsection{Empirical Convergence Rates of the EM Algorithm} \label{sec_ECEM}
\textbf{Initialization.} We assume that the MLEs $\widehat{G}_n$ (\wrt the standard softmax gating function) and $\widetilde{G}_n$ (\wrt the modified softmax gating functions) have $k=3$ components. Next, we initialize fitted parameters in the same fashion as those in Appendix~\ref{appendix:regime_1}. With the sample size $n=10^4$, we run the EM algorithm for $N=200$ iterations, and compute the negative log-likelihood value at each iteration.
\begin{figure}[!ht]
    \centering
    \includegraphics[scale=.49]{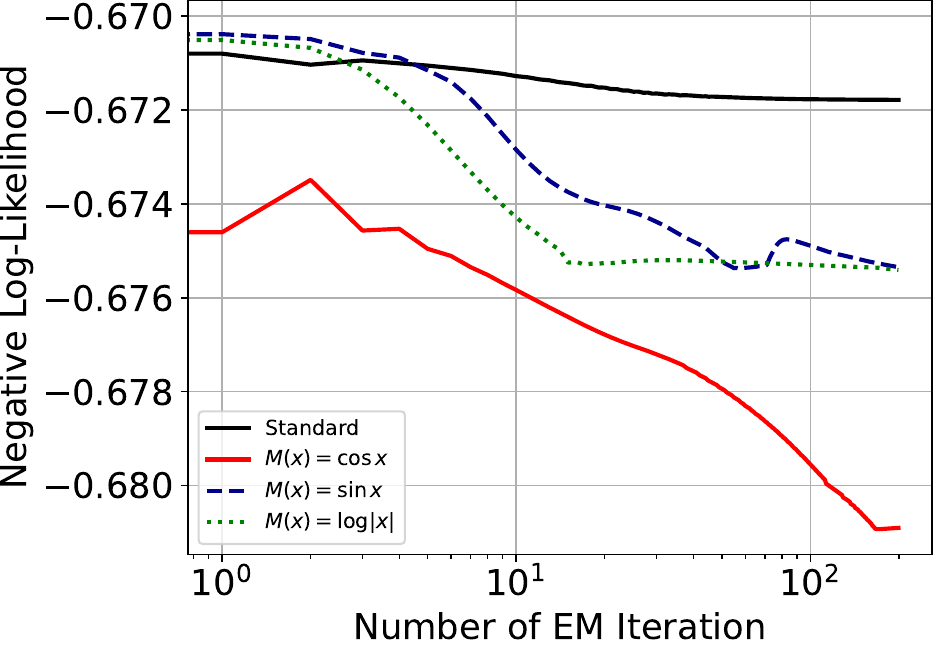}
    
    \caption{Empirical convergence rates of the EM algorithm with the standard softmax gating function and three different modified softmax gating functions with $M(X)\in\{\sin(X),\cos(X),\log(|X|)\}$. The y-axis indicates the negative log-likelihood, while the x-axis illustrates the number of EM iterations.}
    \label{fig:modified_softmax}
\end{figure}

\textbf{Negative Log-likelihood.} It can be seen from Figure~\ref{fig:modified_softmax} that the negative log-likelihood corresponding to the modified softmax gating function with $M(X)=\cos(X)$ experiences a sharp drop after 200 iterations. Meanwhile, those for $M(X)=\sin(X)$ and $M(X)=\log(|X|)$ decline at a nearly same rate but slower than that for $M(X)=\cos(X)$. On the other hand, the negative log-likelihood corresponding to the standard softmax gating function remains almost unchanged. Those observations suggest that it would take a very long time for the MLE $\widehat{G}_n$ (\wrt the standard softmax gating function) to converge to the true mixing measure $G_*$. By contrast, if we use the modified softmax gating functions with $M(X)\in\{\sin(X),\cos(X),\log(|X|)\}$, the convergence rates of the corresponding MLE $\widetilde{G}_n$ to $G_*$ would be substantially faster. As a consequence, this figure highlights the advantages of using the modified softmax gating functions over the standard softmax gating function in the parameter estimation of the multinomial logistic MoE model.

\end{document}